%% file: main.tex
\documentclass[conference]{IEEEtran}

\usepackage{xcolor}
\usepackage[ruled,vlined]{algorithm2e}
\usepackage{multirow}
\usepackage[printonlyused]{acronym}
\usepackage{amsmath,amssymb,amsfonts,amsthm}
\usepackage{algorithmic}
\usepackage{graphicx}
\usepackage{balance}
\usepackage{mathtools}
\usepackage{booktabs}
\usepackage{textcomp}
\usepackage{tikz}
\usepackage{enumerate}
\usepackage{enumitem,kantlipsum}
\usepackage{soul}
\usepackage{colortbl}
\usepackage{pifont}
\usepackage{cancel}
\usepackage{authblk}
\newcommand{\cmark}{\ding{52}}
\newcommand{\xmark}{\ding{56}}
\usepackage{array}
\def\BibTeX{{\rm B\kern-.05em{\sc i\kern-.025em b}\kern-.08em
    T\kern-.1667em\lower.7ex\hbox{E}\kern-.125emX}}

\newcommand*\emptycirc[1][1ex]{\tikz\draw[thick] (0,0) circle (#1);} 
\newcommand*\halfcirc[1][1ex]{%
  \begin{tikzpicture}
  \draw[fill] (0,0)-- (90:#1) arc (90:270:#1) -- cycle ;
  \draw[thick] (0,0) circle (#1);
  \end{tikzpicture}}
\newcommand*\fullcirc[1][1ex]{\tikz\fill (0,0) circle (#1);} 

\usepackage[]{hyperref}
\hypersetup{colorlinks=true, unicode=true, linkcolor=purple, citecolor=blue}

\definecolor{t0}{rgb}{0.0, 0.5, 0.5}
\definecolor{t1}{rgb}{0.007, 0.495, 0.497}
\definecolor{t2}{rgb}{0.015, 0.49, 0.494}
\definecolor{t3}{rgb}{0.022, 0.485, 0.491}
\definecolor{t4}{rgb}{0.03, 0.48, 0.488}
\definecolor{t5}{rgb}{0.037, 0.475, 0.485}
\definecolor{t6}{rgb}{0.045, 0.47, 0.482}
\definecolor{t7}{rgb}{0.052, 0.465, 0.479}
\definecolor{t8}{rgb}{0.059, 0.46, 0.476}
\definecolor{t9}{rgb}{0.067, 0.455, 0.473}
\definecolor{t10}{rgb}{0.074, 0.45, 0.47}
\definecolor{t11}{rgb}{0.082, 0.446, 0.468}
\definecolor{t12}{rgb}{0.089, 0.441, 0.465}
\definecolor{t13}{rgb}{0.097, 0.436, 0.462}
\definecolor{t14}{rgb}{0.104, 0.431, 0.459}
\definecolor{t15}{rgb}{0.111, 0.426, 0.456}
\definecolor{t16}{rgb}{0.119, 0.421, 0.453}
\definecolor{t17}{rgb}{0.126, 0.416, 0.45}
\definecolor{t18}{rgb}{0.134, 0.411, 0.447}
\definecolor{t19}{rgb}{0.141, 0.406, 0.444}
\definecolor{t20}{rgb}{0.149, 0.401, 0.441}
\definecolor{t21}{rgb}{0.156, 0.396, 0.438}
\definecolor{t22}{rgb}{0.163, 0.391, 0.435}
\definecolor{t23}{rgb}{0.171, 0.386, 0.432}
\definecolor{t24}{rgb}{0.178, 0.381, 0.429}
\definecolor{t25}{rgb}{0.186, 0.376, 0.426}
\definecolor{t26}{rgb}{0.193, 0.371, 0.422}
\definecolor{t27}{rgb}{0.2, 0.366, 0.419}
\definecolor{t28}{rgb}{0.208, 0.361, 0.416}
\definecolor{t29}{rgb}{0.215, 0.356, 0.413}
\definecolor{t30}{rgb}{0.223, 0.351, 0.41}
\definecolor{t31}{rgb}{0.23, 0.347, 0.408}
\definecolor{t32}{rgb}{0.238, 0.342, 0.405}
\definecolor{t33}{rgb}{0.245, 0.337, 0.402}
\definecolor{t34}{rgb}{0.252, 0.332, 0.399}
\definecolor{t35}{rgb}{0.26, 0.327, 0.396}
\definecolor{t36}{rgb}{0.267, 0.322, 0.393}
\definecolor{t37}{rgb}{0.275, 0.317, 0.39}
\definecolor{t38}{rgb}{0.282, 0.312, 0.387}
\definecolor{t39}{rgb}{0.29, 0.307, 0.384}
\definecolor{t40}{rgb}{0.297, 0.302, 0.381}
\definecolor{t41}{rgb}{0.304, 0.297, 0.378}
\definecolor{t42}{rgb}{0.312, 0.292, 0.375}
\definecolor{t43}{rgb}{0.319, 0.287, 0.372}
\definecolor{t44}{rgb}{0.327, 0.282, 0.369}
\definecolor{t45}{rgb}{0.334, 0.277, 0.366}
\definecolor{t46}{rgb}{0.342, 0.272, 0.363}
\definecolor{t47}{rgb}{0.349, 0.267, 0.36}
\definecolor{t48}{rgb}{0.356, 0.262, 0.357}
\definecolor{t49}{rgb}{0.364, 0.257, 0.354}
\definecolor{t50}{rgb}{0.371, 0.252, 0.351}
\definecolor{t51}{rgb}{0.379, 0.248, 0.349}
\definecolor{t52}{rgb}{0.386, 0.243, 0.346}
\definecolor{t53}{rgb}{0.394, 0.238, 0.343}
\definecolor{t54}{rgb}{0.401, 0.233, 0.34}
\definecolor{t55}{rgb}{0.408, 0.228, 0.337}
\definecolor{t56}{rgb}{0.416, 0.223, 0.334}
\definecolor{t57}{rgb}{0.423, 0.218, 0.331}
\definecolor{t58}{rgb}{0.431, 0.213, 0.328}
\definecolor{t59}{rgb}{0.438, 0.208, 0.325}
\definecolor{t60}{rgb}{0.446, 0.203, 0.322}
\definecolor{t61}{rgb}{0.453, 0.198, 0.319}
\definecolor{t62}{rgb}{0.46, 0.193, 0.316}
\definecolor{t63}{rgb}{0.468, 0.188, 0.313}
\definecolor{t64}{rgb}{0.475, 0.183, 0.31}
\definecolor{t65}{rgb}{0.483, 0.178, 0.307}
\definecolor{t66}{rgb}{0.49, 0.173, 0.304}
\definecolor{t67}{rgb}{0.498, 0.168, 0.301}
\definecolor{t68}{rgb}{0.505, 0.163, 0.298}
\definecolor{t69}{rgb}{0.512, 0.158, 0.295}
\definecolor{t70}{rgb}{0.52, 0.153, 0.292}
\definecolor{t71}{rgb}{0.527, 0.149, 0.29}
\definecolor{t72}{rgb}{0.535, 0.144, 0.287}
\definecolor{t73}{rgb}{0.542, 0.139, 0.284}
\definecolor{t74}{rgb}{0.55, 0.134, 0.281}
\definecolor{t75}{rgb}{0.557, 0.129, 0.278}
\definecolor{t76}{rgb}{0.564, 0.124, 0.274}
\definecolor{t77}{rgb}{0.572, 0.119, 0.271}
\definecolor{t78}{rgb}{0.579, 0.114, 0.268}
\definecolor{t79}{rgb}{0.587, 0.109, 0.265}
\definecolor{t80}{rgb}{0.594, 0.104, 0.262}
\definecolor{t81}{rgb}{0.601, 0.099, 0.259}
\definecolor{t82}{rgb}{0.609, 0.094, 0.256}
\definecolor{t83}{rgb}{0.616, 0.089, 0.253}
\definecolor{t84}{rgb}{0.624, 0.084, 0.25}
\definecolor{t85}{rgb}{0.631, 0.079, 0.247}
\definecolor{t86}{rgb}{0.639, 0.074, 0.244}
\definecolor{t87}{rgb}{0.646, 0.069, 0.241}
\definecolor{t88}{rgb}{0.653, 0.064, 0.238}
\definecolor{t89}{rgb}{0.661, 0.059, 0.235}
\definecolor{t90}{rgb}{0.668, 0.054, 0.232}
\definecolor{t91}{rgb}{0.676, 0.05, 0.23}
\definecolor{t92}{rgb}{0.683, 0.045, 0.227}
\definecolor{t93}{rgb}{0.691, 0.04, 0.224}
\definecolor{t94}{rgb}{0.698, 0.035, 0.221}
\definecolor{t95}{rgb}{0.705, 0.03, 0.218}
\definecolor{t96}{rgb}{0.713, 0.025, 0.215}
\definecolor{t97}{rgb}{0.72, 0.02, 0.212}
\definecolor{t98}{rgb}{0.728, 0.015, 0.209}
\definecolor{t99}{rgb}{0.735, 0.01, 0.206}
\definecolor{t100}{rgb}{0.75, 0.0, 0.2}

\begin{document}
\newtheorem{theorem}{Theorem}[section]
\newtheorem{proposition}{Proposition}[section]
\newtheorem{corollary}[proposition]{Corollary}

\makeatletter
\renewcommand\AB@affilsepx{, \protect\Affilfont}
\makeatother
\renewcommand\Authands{, }
\renewcommand\Affilfont{\normalsize}

\title{The Adaptive Arms Race: Redefining Robustness in AI Security}

\author[*]{Ilias Tsingenopoulos}
\author[*]{Vera Rimmer}
\author[*]{Davy Preuveneers}
\author[$\dag$]{Fabio Pierazzi}
\author[$\dag$]{Lorenzo Cavallaro}
\author[*]{Wouter Joosen}

\affil[*]{KU Leuven}
\affil[$\dag$]{University College London}

\maketitle

\begin{abstract}
Despite considerable efforts on making them robust, real-world AI-based systems remain vulnerable to decision based attacks, as definitive proofs of their operational robustness have so far proven intractable.
Canonical robustness evaluation relies on adaptive attacks, which leverage complete knowledge of the defense and are tailored to bypass it.
This work broadens the notion of adaptivity, which we employ to enhance both attacks and defenses, showing how they can benefit from mutual learning through interaction.
We introduce a framework for adaptively optimizing black-box attacks and defenses under the competitive game they form.
To assess robustness reliably, it is essential to evaluate against realistic and worst-case attacks.
We thus enhance attacks and their evasive arsenal \emph{together} using \ac{RL}, apply the same principle to defenses, and evaluate them first independently and then jointly under a multi-agent perspective.

We find that active defenses, those that dynamically control system responses, are an essential complement to model hardening against decision-based attacks; that these defenses can be circumvented by adaptive attacks, something that elicits defenses being adaptive too.
Our findings, supported by an extensive theoretical and empirical investigation, confirm that adaptive adversaries pose a serious threat to black-box AI-based systems, rekindling the proverbial arms race.
Notably, our approach outperforms the state-of-the-art black-box attacks \textit{and} defenses, while bringing them together to render effective insights into the robustness of real-world deployed ML-based systems.
\end{abstract}


\setlength{\abovedisplayskip}{4pt}
\setlength{\belowdisplayskip}{4pt}

\input{acronyms.tex}

\input{introduction.tex}
\input{background.tex}
\input{approach.tex}
\input{threat.tex}
\input{evaluation.tex}
\input{discussion.tex}

\bibliographystyle{IEEEtran}
\bibliography{references}

\input{appendix.tex}

\end{document}

%% file: acronyms.tex
\acrodef{ML}[ML]{machine learning}
\acrodef{DL}[DL]{deep learning}
\acrodef{RL}[RL]{reinforcement learning}
\acrodef{AML}[AML]{adversarial machine learning}
\acrodef{MDP}[MDP]{Markov Decision Process}
\acrodef{DNN}[DNN]{deep neural network}
\acrodef{RNN}[RNN]{recurrent neural network}
\acrodef{SVM}[SVM]{Support Vector Machine}
\acrodef{NLP}[NLP]{natural language processing}
\acrodef{GAN}[GAN]{generative adversarial network}

\acrodef{gan}[GAN]{generative adversarial network}
\acrodef{IID}[IID]{independent and identically distributed}
\acrodef{OOD}[OOD]{out-of-distribution}
\acrodef{HW}[HW]{Hardwaere}
\acrodef{SW}[SW]{Software}
\acrodef{AE}[AE]{Adversarial Example}
\acrodef{FPS}[FPS]{frames per second}
\acrodef{PPO}[PPO]{Proximal Policy Optimization}

\acrodef{wrt}[\emph{w.r.t.}]{with respect to}
\acrodef{st}[\emph{s.t.}]{such that}

%% file: introduction.tex
\section{Introduction}

AI models are predominantly trained, validated, and deployed with little regard to their correct functioning under adversarial activity, often leaving safety and security considerations as an afterthought.
Adversarial contexts further aggravate the typical generalization challenges that these models face with threats beyond model evasion (misclassification), like model extraction, model inversion, and model poisoning \cite{he2020towards}.
At the same time, the systems these models are components of often expose interfaces that can be queried and used as adversarial ``instructors'', like in constructing adversarial malware against existing AI-based malware detection~\cite{anderson2018learning, demetrio2021functionality}.
Focusing on adversarial examples for model evasion, the most reliable mitigation to date is adversarial training~\cite{madry2017towards, wang2019convergence}, an approach not without limitations as these models often remain irreducibly vulnerable at deployment, particularly against black-box, decision-based attacks~\cite{brendel2018decision, chen2020hopskipjumpattack, yan2020policy}.
Nevertheless, all such attacks exhibit a behavior at-the-interface that can be described as adversarial itself, a generalization that subsumes adversarial examples and opens a path towards novel defenses and mitigations.

Adversarial behavior is a temporal extension of adversarial examples, perhaps not malicious or harmful in isolation, yet part of an attack as it unfolds over time; it is also the canonical description of adversarial examples in domains like dynamic malware analysis and adversarial \ac{RL} \cite{tsingenopoulos2022adaptive, gleave2020adversarial}.
Aside from making the underlying models more robust, this behavior can be countered as such, rather than relying on hardened models exclusively.
As AI models cannot update their decision boundary in an online manner and in response to adversarial activity on their interface, there \emph{has} to be a complement to model hardening: for instance \emph{active} defenses such as rejection or misdirection \cite{barbero2022transcending, sengupta2020multi, chen2020stateful}.

In this study we identify and address a crucial gap: evaluating the robustness of defenses against oblivious, non-adaptive, and therefore suboptimal attackers, renders any results unreliable~\cite{tramer2020adaptive,croce2020reliable}.
The key observation we make is that robustness \textit{must} account for the ability of the adversary to adapt while interacting with the model.
To that end, we expand the conventional notion of adaptive, from \emph{adapted} attacks that have an empirical configuration to bypass the defense, to include the capability to \emph{self-adapt}, where attacks adapt their parameters and evasive actions \textit{together}, based on how the target model and its defenses respond~\cite{aastrom1995adaptive}.
We demonstrate theoretically and empirically how self-adaptive attacks can use \ac{RL} to modify their policies to become both optimal \emph{and} evade active detection.
Notably, this can be performed in a gradient-based manner even in fully black-box contexts~[\ref{th:epg}], and is a capability that \emph{properly reflects} the level of adversarial threat and in that way does not overestimate the empirical robustness; real attackers will compute gradients after all~\cite{apruzzese2023real}.

Through proper threat modeling and self-adaptation, attacks can reach their full potential, enabling the development of effective defensive policies.
To frame the need for adaptive evaluations in \ac{AML} differently: a defense can be considered trustworthy only if it is evaluated against an optimal adversary.
This mutual interdependence underscores the necessity for \emph{both} attacks and defenses to be self-adaptive, thereby establishing the competitive, zero-sum dynamic inherent in their interaction.
In this work, we examine robustness from both perspectives: first, how to fully optimize decision-based attacks, and second, how to devise reliable countermeasures.
We explore both offensive and defensive strategies in depth, and make the following key contributions:

\begin{enumerate}[wide, labelindent=5pt, noitemsep, nolistsep, label=\textbf{\arabic*}.]
\item We demonstrate that active defenses against decision-based attacks are a \textit{necessary} but \textit{insufficient} complement to model hardening.
Active defenses are inevitably bypassed by self-adaptive attackers however, and necessitate \textit{self-adaptive} defenses too.
\item To facilitate reasoning on adaptive attacks and defenses, we introduce a unified framework called ``Adversarial Markov Games'' (\textbf{AMG}).
We demonstrate how adversaries can optimize their attack policy and evade active detection \textit{at the same time}; as a counter, we develop a novel active defense and employ \ac{RL} agents to \textit{adapt} and optimize both.
\item In an extensive empirical evaluation on image classification and across a wide set of adversarial scenarios, we validate our theoretical analysis and show that self-adaptation with RL \textit{outperforms} vanilla black-box attacks, model hardening defenses like adversarial training, and notably \textbf{both} the state-of-the-art adaptive attack (OARS \cite{feng2023stateful}) and stateful defense (Blacklight~\cite{li2022blacklight}).
This supports self-adaptation as an \emph{essential} component when evaluating robustness to black-box attacks.
\item For reproducibility, and to facilitate further research, we open-source our code\footnote{https://anonymous.4open.science/r/AMG-AD16}.
\end{enumerate}

Our work highlights that in the domain of black-box \ac{AML}, robust evaluations \textit{should} go a step further than adapting attacks: both attacks and defenses should have the capability to optimize their strategies through interaction and in direct response to other agency in their environment.
The remainder of the paper is structured as follows:
Section \ref{sec:background} provides the necessary background on the domain and reviews the related work.
Section \ref{sec:approach} introduces and motivates our theoretical analysis of robustness under decision-based attacks.
Section \ref{sec:threat} explains the threat model and the concrete design choices.
In Section \ref{sec:evaluation} we elaborate on our experimentation and analyze our results.
We conclude with Section \ref{sec:discussion} where we discuss key insights, limitations and challenges.

%% file: background.tex
\section{Preliminaries}
\label{sec:background}
In this work, we focus on the category of adversarial attacks known as \textbf{decision-based}, a subset of \text{query-based} attacks that operate solely on the \textbf{hard-label} outputs of the model and are a highly realistic and pervasive threat in AI-based cybersecurity environments.
Despite the lack of the closed-form expression of the model under attack, given enough queries their effectiveness can match the one of white-box techniques~\cite{carlini2017towards, croce2020reliable}.

\subsection{Attacks \& Mitigations}
While adversarial attacks have been extensively researched in both white and black-box contexts, defenses have predominantly focused on white-box~\cite{madry2017towards, wang2019convergence}.
As the black-box setting discloses considerably less information, a seemingly intuitive conclusion is that white-box defenses should suffice for the black-box case too.
Yet black-box attacks like~\cite{brendel2018decision, chen2020hopskipjumpattack} have shown to be highly effective against a wide range of defenses like \emph{gradient masking}~\cite{athalye2018obfuscated}, \emph{preprocessing}~\cite{qin2021random, byun2022effectiveness}, and \emph{adversarial training}~\cite{madry2017towards}.
The vast majority of adversarial defenses provide either limited robustness or are eventually evaded by adapted attacks \cite{tramer2020adaptive}.
Characteristically, preprocessing defenses are identified and bypassed by expending queries for reconnaissance~\cite{sitawarin2022preprocessors}.

The partial exception to this rule is adversarial training.
Given dataset $D = {(x_i, y_i)}^{n}_{i=1}$ with classes $C$ where $x_i \in \mathbb{R}^d$ is a clean example and $y_i \in {1,..., C}$ is the associated label, the objective of adversarial training is to solve the following \emph{min-max} optimization problem:

\begin{equation}
    \underset{\phi}{\operatorname{min}} \mathbb{E}_{i\sim D} \underset{\Vert \delta_i \Vert_{L_p} \leq \epsilon}{\operatorname{max}} \; \mathcal{L}(h_{\phi}(x_i + \delta_i), y_i)
\label{eqn:adv_train}
\end{equation}

\noindent where $x_i + \delta_i$ is an adversarial example of $x_i$, $h_\phi : \mathbb{R^d} \rightarrow \mathbb{R^C}$ is a hypothesis function and $\mathcal{L}(h_\phi(x_i + \delta_i), y_i)$ is the loss function for the adversarial example $x_i + \delta_i$.
The inner maximization loop finds an adversarial example of $x_i$ with label $y_i$ for a given $L_p$-norm (with $L_p \in \{0,1,2,\inf\}$), such that $\Vert \delta_i\Vert_{l} \leq \epsilon$ and $h_\phi(x_i + \delta_i) \neq y_i$.
The outer loop is the ordinary minimization task, typically solved with stochastic gradient descent.
While the convergence and robustness properties of adversarial training have been investigated through the computation of the saddle point and by interleaving normal and adversarial training \cite{wang2019convergence}, the min-max principle is conspicuous: minimize the possible loss for a worst-case (max) scenario.

\subsection{Stateful Defenses}
Decision-based attacks possess properties that can be valuable for devising defenses against them, \textit{in addition} to adversarial training.
One such property is their intrinsic sequentiality: by following a policy toward the optimal adversarial example, the generated candidates are correlated.
Note that this might not hold for the queries \textit{themselves}, as the adversary may apply transformations that the model is invariant to, such as the query blinding strategy in Chen et al.~\cite{chen2020stateful}.
This work is the first to employ a \emph{stateful} defense against query-based attacks.
Another stateful defense is PRADA \cite{juuti2019prada}, devised against model extraction but effective against evasion too.
These approaches assume however that queries can be consistently linked (via metadata like IP or account, cf.~\autoref{tab:comparison}) to uniquely identifiable actors -- who also exhibit limited to no collaboration -- so that a query buffer can be built for each.

This limitation, together with the scalability issues, was recently addressed in the Blacklight defense, by employing hashing and quantization~\cite{li2022blacklight}.
Blacklight remains a similarity-based defense, thus vulnerable to circumvention if an adversary can find a query generation policy that preserves the attack functionality while evading detection.
OARS achieved this by adapting existing attacks through the rejection signal Blacklight returns~\cite{feng2023stateful}.
Ultimately, any (stateful) defense has to balance the trade-off between robust and clean accuracy; as we demonstrate in this work, this trade-off can be measured reliably only if the attacker is properly adaptive.

\subsection{On Being Adaptive}
\label{sub:adaptive}

The correct way to evaluate any proposed defense is against \emph{adaptive} attacks, that is with explicit knowledge of the concrete mechanisms of a defense~\cite{tramer2020adaptive}.
In computer security this is known as the stipulation that security through obscurity does not work, as the robustness of defenses should not rely on keeping their mechanism secret.
If model hardening -- for instance by adversarial training -- is the defensive counterpart to white-box attacks, active defenses like stateful detection are the counterpart to decision-based attacks, and as we will further demonstrate, also the \emph{necessary} complement to hardening a model against them.

At the same time, the level of threat that attacks pose is often unclear or not thoroughly evaluated.
Previous work has demonstrated that the loss functions and parameters of attacks are often suboptimal, leading to \textit{underestimating} their performance and thus \textit{overestimating} the claimed degree of robustness~\cite{croce2020reliable, pintor2022indicators}.
This underestimation is further aggravated in decision-based contexts, where the attacker is largely oblivious of any preprocessing or active defenses the black-box system might have.
The true performance of attacks therefore rests on the ability to adapt their operation policy and their evasive capabilities \textit{in tandem}.

\begin{figure}
    \centering
    \includegraphics[width=0.49\textwidth]{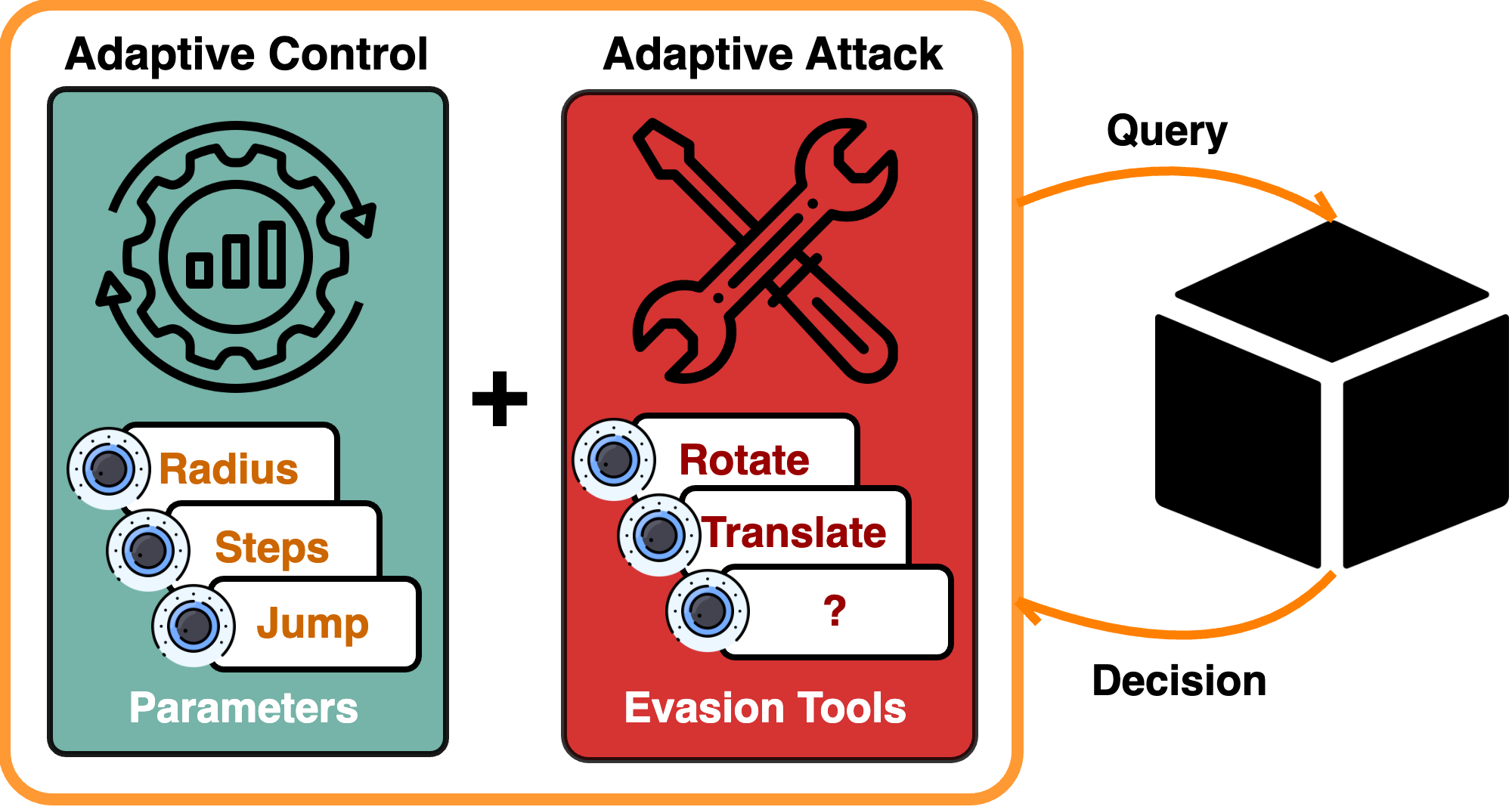}
    \caption{\small{In \ac{AML}, adaptive attacks are those with the capabilities (knobs) to bypass a defense; adaptive control is rather the precise tuning of all the known knobs. Against black-box systems, we can reformulate adaptive so that it signifies \textbf{both}. For instance in HSJA~\cite{chen2020hopskipjumpattack}, radius, steps, and jump are parameters of the attack, while rotate and translate are transformations that can evade a similarity-based defense.}}
    \label{fig:adaptivity}
\end{figure}

In \ac{AML}, ``adaptive'' by convention refers to attacks with full knowledge of how a defense works and the tools to bypass it; we denote such attacks as \textbf{adapted}.
In this work we expand the term to include \textit{adaptive control}, defined as the ability of a system to \textbf{self-adapt}: \emph{automatically} reconfigure itself in response to changes in the dynamics of the environment in order to achieve optimal behavior~\cite{aastrom1995adaptive}.
We use adaptive control in the sense ``attack optimization'' is used by Pintor et al.~\cite{pintor2022indicators}, but here for black-box systems.
What is to be controlled is typically known in advance and well-defined.
However, the moment we consider adaptive evaluations, \textit{new} controls are directly implied: in a similarity-based defense for instance, such controls would be transformations to the input that the model is invariant to.
To flesh out the twofold meaning of adaptive, one has to \textit{both} invent new knobs~\cite{hofstadter2008metamagical} -- the conventional understanding of adaptive, intractable to automate yet -- \textit{and} dynamically control their correct configuration that would lead to the optimal result (self-adaptive).
We conceptualize this more general definition of adaptive, essential for having accurate evaluations against decision-based attacks, in~\autoref{fig:adaptivity}.

\begin{table*}[!ht]
\small
\caption{\small{Positioning of our work relative to prominent decision-based attacks and defenses and their individual properties. While prior works focus exclusively on offense or defense, ours unifies and reasons from both perspectives.}} 
\centering
\begin{tabular}{r cccc cccc}
\toprule
\multirow{2}{*}{\textbf{Work}} & \multicolumn{4}{c}{\textbf{\textcolor{purple!80}{Offensive}}} & \multicolumn{4}{c}{\textbf{\textcolor{teal!80}{Defensive}}} \\
\cmidrule(lr){2-5} \cmidrule(lr){6-9}
& \textbf{Optimized} & \textbf{Evasive} & \textbf{Adaptive} & \textbf{$\neg$Rejection} & \textbf{Active} & \textbf{Adaptive} & \textbf{$\neg$Metadata} &\textbf{Misdirection} \\
\midrule
\textbf{\textcolor{purple!80}{Boundary (2018)}}  \cite{brendel2018decision} & \emptycirc & \emptycirc & \emptycirc & \fullcirc & \emptycirc & \emptycirc & --- & --- \\
\textbf{\textcolor{purple!80}{BAGS (2018)}} \cite{brunner2019guessing} & \emptycirc & \emptycirc & \emptycirc & \fullcirc & \emptycirc & \emptycirc & --- & --- \\
\textbf{\textcolor{purple!80}{HSJA (2020)}} \cite{chen2020hopskipjumpattack} & \fullcirc & \emptycirc & \emptycirc & \fullcirc & \emptycirc & \emptycirc & --- & --- \\
\textbf{\textcolor{purple!80}{OARS (2023)}} \cite{feng2023stateful} & \fullcirc & \fullcirc & \halfcirc & \emptycirc & \fullcirc & \emptycirc & --- & --- \\
\midrule
\textbf{\textcolor{teal!80}{Adv. Training (2017)}} \cite{madry2017towards} & \fullcirc & \emptycirc & \emptycirc & --- & \emptycirc & \emptycirc & --- & ---\\
\textbf{\textcolor{teal!80}{Stateful (2020)}} \cite{chen2020stateful} & \fullcirc & \fullcirc & \emptycirc & \emptycirc & \fullcirc & \emptycirc & \emptycirc & \emptycirc \\
\textbf{\textcolor{teal!80}{Blacklight (2022)}}  \cite{li2022blacklight} & \fullcirc & \fullcirc & \emptycirc & \emptycirc & \fullcirc & \emptycirc & \fullcirc & \emptycirc \\
\midrule
\textbf{\textcolor{orange!80}{Our work}}  & \fullcirc & \fullcirc & \fullcirc & \fullcirc & \fullcirc & \fullcirc & \fullcirc & \fullcirc \\
\bottomrule
\label{tab:comparison}
\end{tabular}
\end{table*}

\subsection{Research Gap}
Prior work has focused on \textit{adapted} attacks, which incorporate general knowledge of any defenses, then empirically configured to evade it~\cite{carlini2017towards, chen2020hopskipjumpattack, brendel2018decision}.
Defenses also follow the same adapted paradigm of empirically defined and fixed operation~\cite{chen2020stateful, li2022blacklight}.
Our observation is that neither of them are formalized or performed in a fully adaptive manner, that is in response to how they influence their environment and with respect to other adaptive agents in it, with clear limitations when the latter is a given, e.g. in cybersecurity.
To bridge this gap, we provide a theoretical analysis and an empirical study of existing and novel methodologies adapting through direct interaction with their environment, denoting them as \textbf{self-adaptive}.

Our work builds on a long line of prior research that focuses on both sides of the competition between adversaries and defenses.
\textbf{Carlini and Wagner}~\cite{carlini2017towards} show that evaluating existing attacks out-of-the-box is insufficient and that adapted white-box attackers can break defensive distillation.
\textbf{Bose et~al.}~\cite{bose2020adversarial} propose Adversarial Examples Games, a zero-sum game between a white-box attacker and a local surrogate of the target model family.
At the equilibrium the attacker can generate adversarial examples that have a high success rate against models from the same family, constituting a zero-query, non-interactive approach for generating transferable adversarial examples.
\textbf{Pal et al.}~\cite{pal2020game} propose a game-theoretic framework for studying white-box attacks and defenses that occur in equilibrium.
\textbf{Feng et al.}~\cite{feng2023stateful} introduce \textbf{OARS}: adaptive versions of existing attacks that bypass \textbf{Blacklight}~\cite{li2022blacklight}, the state-of-the-art stateful defense.
To function, OARS presupposes the rejection signal that a defense like Blacklight returns; a strong assumption that as we show in this work does not have to hold for stateful defenses.
As we demonstrate in~\autoref{sec:evaluation} and~\autoref{tab:result3}, Blacklight can be bypassed without assuming rejection, while the novel stateful defense we introduce can fully withstand the OARS adaptive attack.

As the most relevant and representative threat against real-world AI systems, in this work we scope on decision-based, interactive attacks and defenses.
We contribute a theoretical and practical framework for self-adaptation, under which the full extent of the offensive and thus also the defensive potential is properly assessed.
In the remainder of the paper the term \textbf{``adaptive''} subsumes adaptive control, and is used interchangeably with \textbf{``self-adaptive''}.
For what is conventionally known as adaptive evaluations in \ac{AML}, we use the term \textbf{``adapted''}.
To facilitate comparison, in~\autoref{tab:comparison} we highlight the most important aspects of our work as the synthesis of adaptive black-box attacks and defenses in a unified framework, and situate it with respect to prominent and state-of-the-art works in \ac{AML}.
Note the importance for an attack to function \textit{without} access to a rejection signal, and respectively for a defense to function \textit{without} access to query metadata like IP addresses or accounts.


%% file: approach.tex
\section{Theoretical Framework}
\label{sec:approach}
In this section, we abstract through the individual properties of decision-based attacks and defenses to extract more general insights than a purely empirical study would render.
To investigate how robust real-world systems are to evasion, two related perspectives are crucial: a) resisting decision-based attacks, and b) adapting attacks and evasive capabilities \textit{together}.
When attacks (and defenses) are evaluated in a non-adaptive manner, in the expansive sense we outlined in \autoref{sub:adaptive}, results are unreliable~\cite{tramer2020adaptive, pintor2022indicators}.
Note that with offensive or defensive methodologies adapting, their environments become non-stationary~\cite{hernandez2017survey}, putting further pressure on the IID foundations that ML builds on.
To understand the implications of this adaptation, we perform an analysis of the possible interactions on the interface of an ML-based system, interactions that can be more generally considered as sequential zero-sum games~\cite{littman1994markov, hardt2016strategic, bose2020adversarial}.
In the following sections, introduced terms and notation are highlighted in \textcolor{orange}{orange}.

\subsection{Attacks}
The most compelling threat that deployed ML-based systems face are decision-based black-box attacks, where no access is assumed to the model or its parameters, only the capacity to submit queries and receive hard-label responses.
One of the first decision-based attacks was Boundary Attack~\cite{brendel2018decision}, followed by others that improve the overall performance, typically measured as the lowest perturbation achieved for the minimal amount of queries submitted.
Prominent examples are HSJA~\cite{chen2020hopskipjumpattack}, Guessing Smart (BAGS)~\cite{brunner2019guessing}, Sign-Opt~\cite{cheng2019sign}, Policy-driven (PDA)~\cite{yan2020policy}, QEBA~\cite{li2020qeba}, and SurFree~\cite{maho2021surfree}.

White-box attacks like C\&W~\cite{carlini2017towards} do not function in black-box environments, as there is no access to the inference pipeline.
To facilitate optimization, decision-based attacks commonly initialize from a sample belonging to the target class, as it can be considered an adversarial example with an unacceptably large perturbation.
This switch allows the task to be solved continuously, by minimizing the perturbation while always staying on the adversarial side of the boundary.
Decision-based attacks share further common aspects in their function that we can abstract through: given \textbf{starting} and \textbf{original} samples $\color{orange}x_g$ and $\color{orange}x_{c}$ respectively, the goal is to iteratively propose adversarial \textbf{candidates} $\color{orange}x_t$, until the \textbf{distance} $\color{orange}\delta = d(x_t, x_c)$ is minimized.
This process follows different algorithmic approaches that represent different geometrical intuitions; we can describe it more generally by means of a candidate generation policy:

\begin{equation}
    \pi_\theta^{\mathcal{A}} = P\: (x_{t}|x_g, x_{c}, p^{\mathcal{A}}, s^{\mathcal{A}}),
\label{eqn:genpolicy}
\end{equation}

\noindent that given $x_g$ and $x_{c}$, with $p^{\mathcal{A}}$ the \textbf{parameters} and $s^{\mathcal{A}}$ the \textbf{state} of the attack, generates a candidate $x_t$.
As attacks execute over discrete time steps, if we assume that the model always answers the attack procedure can be construed as a \ac{MDP} to be solved, by finding the parameters $\theta$ that minimize $\delta$ for a given number of queries.

Consider now a multinomial image classification model $\mathcal{M}$ under attack, with a discriminant function $\color{orange}F: \mathbb{R}^d \rightarrow \mathbb{R}^m$, that for each input $x \in [0,1]^d$ generates an output $y := \{y \in [0,1]^m |\sum_{c=1}^{m}y_c = 1\}$ -- a probability distribution over the $m$ classes.
As black-box environments prevent access to these probabilities, one can only observe the decision of the classifier $C$ that returns the highest probability class:

\begin{equation}
    C(x) := \operatorname*{arg\,max}_{c \in [m]} F_c(x) = D(F_c(x))
\label{eqn:classifier} 
\end{equation}

\noindent with $D$ being the decision function, here $D = \operatorname*{arg\,max}$.
The goal in targeted attacks is to change the \textbf{decision} $\color{orange}c_g \in [m]$ for a correctly classified example $x$, to a predefined \textbf{target} class $\color{orange}c_o \neq c_g$.
This process can be facilitated through a function $\color{orange}\psi$ which given a perturbed example $x_t$ at step $t$, it returns a binary indicator of success:

\begin{equation}
    \psi(x_t) = \begin{cases}
                    +1 & \text{if}\quad C(x_t) = c_o\\
                    -1 & \text{if}\quad C(x_t) \neq c_o
                \end{cases}
\label{eqn:psi}
\end{equation}

As long as the model responds, $\psi$ can always be evaluated, and constitutes the fundamental mechanism upon which decision-based attacks build.
The adversarial goal can then be described as the following constrained optimization problem:

\begin{equation}
    \operatorname*{min}_{x_t} \: d(x_t,x_c) \quad \text{s.t.} \quad \psi(x_t) = 1,
\label{eqn:opt}
\end{equation}

\noindent where the distance metric $\color{orange}d$ is an $\ell_p$-norm, with $p \in \{0,1,2,\inf\}$.
As the threshold between adversarial and non-adversarial relies on the subjectivity of human perception, this highlights the indefinite nature of adversarial examples, further exemplified in domains where visual proximity is of little importance.
Successful or unsuccessful adversarial examples are therefore delimited by an \textbf{threshold} $\color{orange}\epsilon$ on perturbation, where $d(x,x_t) \leq \epsilon$.

Real-world attacks being black-box does not make them less effective.
For instance, HSJA is guaranteed to converge to a stationary point of~Eq.~\eqref{eqn:opt}.
Given typical $\epsilon$ values for imperceptibility, this results in high attack success rates, even against \textit{adversarially trained} models. 
The limitations of adversarial training against decision-based attacks can be attributed to the out-of-distribution (OOD) nature of adversarial examples, and the saddle point optimization problem of Eq.~\eqref{eqn:adv_train} that make it difficult for algorithms to converge to a global solution.
Furthermore, incorporate decision-based attacks in training is not scalable as it can take orders of magnitude more steps (queries) to produce an adversarial example, than white-box attacks which take a few steps (1-50 in e.g. PGD~\cite{madry2017towards}).

Decision-based attacks search for the \textbf{optimal parameters} $\color{orange}\theta^*$ of the generation policy \eqref{eqn:genpolicy}, those that given $x^{i}_{c}$, with $i$ denoting the i-th adversarial episode, minimize Eq.~\eqref{eqn:opt} in expectation:

\begin{equation}
    \operatorname*{arg\,min}_{\theta} \mathbb{E}[\sum_{i=1}^{N}d(x^i_b,x^i_c)], \quad \text{s.t.} \quad \psi(x^i_b) = 1,
\label{eqn:rew}
\end{equation}

\noindent where $\color{orange}x^i_b$ is the \textbf{best} adversarial example generated by policy $\pi_\theta^{\mathcal{A}}$ during episode $i$.
Given the dimensionality of the input, it can be intractable to learn a policy that modifies the feature space directly~\cite{pierazzi2020intriguing}; CIFAR-10, for instance, has more than 3K features to perturb.

In AI-based systems, the best practice is to freeze the model after validation so that no novel issues are introduced by retraining: for all queries $x_t$ submitted during an attack session, we can therefore assume that $F_0 = F_1 = ... = F_t, \forall t$.
While this is representative of real-world settings, it also enables adversaries to discover adversarial examples that were not identified beforehand.
Consequently, while model-hardening through adversarial training is \emph{necessary}, it can also be \emph{insufficient} against decision-based attacks like HSJA.


\begin{proposition}
Let $F_c$ denote the discriminant function of an adversarially trained model $\mathcal{M}$, and let $C(x) = D(F_c(x))$ denote its classifier. Then in HSJA, to satisfy $\mathbb{E}[\sum_{i=1}^N d(x_b^i, x_c^i)] \geq \epsilon$ it is necessary that: (a) $D \neq \operatorname*{arg\,max}$, and (b) context $\tau$ exists s.t. for some query $x_t$, $D(F_c(x_t)) \neq D'(\tau, F_c(x_t))$, where $D'$ is a stateful extension of $D$.
\label{prop:one}
\end{proposition}

Intuitively, HSJA operates in 3 stages which repeat: a binary search that puts $x_t$ on the decision boundary, a gradient estimation step, and projection step along the estimated gradient.
If the model \emph{always} responds truthfully, the adversary will be able to accurately perform all these steps and converge to the optimal adversarial; without loss of generality, we can extend this intuition to other decision-based attacks which navigate the boundary.
Secondly, the model should be able to distinguish between two, otherwise identical, queries, when one is part of an attack and the other is not, a capability achievable through statefulness; see Appendix~\ref{apx:proofs} for the proof.

\subsection{Defenses}

Proposition~\ref{prop:one} suggests that alternative classification policies are necessary in the presence of decision-based attacks, e.g. classification with rejection or intentional misdirection.
Rejection has been realized in the form of conformal prediction, where model predictions are sets of classes including the empty one, or learning with rejection~\cite{barbero2022transcending, cortes2016learning}; while misdirection has emerged as a technique in adversarial RL and cybersecurity domains~\cite{gleave2020adversarial, sengupta2020multi}.
While adversarially training the discriminant function $F$ empirically shows some degree of robustness to decision-based attacks, the manner in which the model responds has a complementary potential.
The gap between the empirical and theoretically achievable robustness is the source for an active defense \textit{distinct} from model hardening.
Active defenses have direct implications on attacks themselves however.
Let us now assume an agent carrying out an \textbf{active defense policy}:

\begin{equation}
\pi_\phi^{\mathcal{D}} = P(\alpha_t|x_t,s^{\mathcal{D}}_t), \: \alpha \in \{0,1\}
\label{eqn:def}
\end{equation}

\noindent with $x_t$ the query, $s^{\mathcal{D}}_t$ the \textbf{state} for the defense as created by past queries, and $\color{orange}\boldsymbol{\alpha}$ the \textbf{binary decision}: for queries deemed adversarial, $\boldsymbol{\alpha}=1$, otherwise $\boldsymbol{\alpha}=0$.
When this policy is stationary, the environment dynamics become stationary in turn, thus besides the adversarial task itself, bypassing the defense can \emph{also} be formulated as an MDP to be solved (\autoref{fig:model}).
In two-player, zero-sum games, the moment an agent follows a stationary policy, it becomes \textit{exploitable} through the reward obtained by an adversary~\cite{timbers2022approximate}.
Active defenses, as consequence of decision-based attacks, entail therefore \textit{adaptive} adversaries.

\begin{proposition}
Against an active defense $\pi_\phi^{\mathcal{D}}$ and for time horizon $T$, a decision-based attack following a non-adaptive candidate generation policy $\pi_t = \pi_\theta^\mathcal{A}, \forall \: t \in [0,T]$ will perform worse in expectation \eqref{eqn:rew}, that is $\mathbb{E}[\sum_{i=1}^{N}d(x^i_b,x^i_c)]^{\mathcal{D}} > \mathbb{E}[\sum_{i=1}^{N}d(x^i_b,x^i_c)]^{\cancel{\mathcal{D}}}$.
\label{prop:two}
\end{proposition}

A proof for BAGS and HSJA is included in Appendix~\ref{apx:proofs}.
An adversary can reason, as a corollary to Proposition~\ref{prop:one}, that such defenses \emph{have to} be in place as it is suboptimal not too.
However, there is a second reason to consider adaptive attacks even in the absence of active defenses, as attack policies are often suboptimal with their default, empirically defined parameters.
Adapting attack policies is essentially the optimization of these parameters, and as an approach has proven very effective in other black-box or expensive-to-evaluate domains, like Neural Architecture Search and Data Augmentation~\cite{zoph2016neural, pham2021autodropout, tsingenopoulos2024train}.
Our results in~\autoref{sec:evaluation} further indicate the correspondence between adaptive and self-optimizing, showing that adaptive consistently outperform non-adaptive attacks, particularly against active defenses.

Consider now an active defense that is based on a similarity or conformal metric.
In the twofold meaning we introduced in \autoref{sub:adaptive}, adaptive attack implies the \textit{capability} to bypass a similarity based defense; adaptive control implies optimization instead, the active tuning of all the available tools to evade the defense \textit{and} minimize the perturbation~(cf. \autoref{fig:adaptivity}).
The updated adversarial objective then is to find the optimal policy that \emph{also} evades detection, and the way to achieve this is by adapting the candidate generation policy \eqref{eqn:genpolicy} itself.
Notably, and despite the black-box and discontinuous nature of the task, this optimization can be \emph{fully} gradient-based.
Decision-based attacks can recover \textbf{gradient-based} solutions to their objective, despite \emph{neither} the active defense \emph{nor} the model itself being accessible in closed-form.
For model $\mathcal{M}$, adversarial queries $x_t$, and active defense $\pi_\phi^{\mathcal{D}}$ making decisions $\alpha_t$, we can thus formulate the following:

\begin{proposition}[Adversarial Policy Gradient]
Given adversarial policy $\pi_\theta^\mathcal{A}$ \eqref{eqn:genpolicy} that generates episodes $\tau_i$ of queries $x_t$, and reward function $r(\tau_i) = \sum_{x_t \in \tau_i} (1 - \alpha_t)$, the optimal evasive policy ${\pi^\mathcal{A}_{\theta^*}}$ is obtained by gradient ascent on the policy's expected reward, $\nabla_\theta \mathbb{E}_{\pi_\theta^{\mathcal{A}}} [r(\tau_i)]$.
\label{th:epg}
\end{proposition}

The proof is included in Appendix~\ref{apx:proofs}.
We thus have established that, \textbf{a)} in the presence of decision-based attacks, active defenses are necessary, yet conditional on adversarial agency they are insufficient and, \textbf{b)} adaptive attacks can become optimal in terms of both evasion and efficiency by observing and adapting to the discrete model decisions.
To complete the puzzle, the last piece is turning active defenses also adaptive.

\begin{corollary}
The active defense achieves its optimal $\pi_{\phi^*}^{\mathcal{D}}$~\eqref{eqn:def}, i.e. maximizing expectation $\mathbb{E}[\sum_{x_t \in \tau_i}P(\alpha_t|x_t,s^{\mathcal{D}}_t)]$, by adapting its policy against the optimal evasive policy ${\pi^\mathcal{A}_{\theta^*}}$.
\label{prop:3}
\end{corollary}

\begin{proof}
Since the game is zero-sum, we may define the defensive policy reward $\rho$ on any trajectory $\tau_i=(x_1,\dots,x_T)$ as $\rho(\tau_i) = \sum_{x_t \in \tau_i} \alpha_t$.
Treating the adversary's policy $\pi_{\theta^*}^\mathcal{A}$ as fixed, we perform gradient ascent on the expected reward $J(\phi)\;=\;\mathbb{E}_{\tau_i\sim(\pi_{\theta^*}^{\mathcal A},\,\pi_{\phi}^{\mathcal D})}\bigl[\rho(\tau_i)\bigr]$.
Under standard smoothness assumptions, this converges to $\phi^*=\arg\max_\phi J(\phi)$, which is precisely the defender’s best response to $\pi_{\theta^*}^{\mathcal A}$.
\end{proof}

\subsection{Adversarial Markov Games}
\label{sec:AMG}
By reasoning on both offensive and defensive capabilities, we highlight why one cannot consider them independently.
As adaptive attacks and defenses are logical consequences of each other, their composition forms a turn-taking competitive game.
A precise game-theoretic formulation requires full knowledge of the environment: its models, players and their utility functions, as well as the permitted interactions and the transition dynamics, something typically intractable in this and other cybersecurity settings.
Model-free methods however can learn optimal (offensive \textit{and} defensive) responses directly through interaction with their environment~\cite{sengupta2020multi, schulman2017proximal}, avoiding the need for explicit modeling or solving the NP-hard bi-level optimization problem of Eq.~\eqref{eqn:adv_train}~\cite{bruckner2011stackelberg}.

To that end, Turn-Taking Partially-Observable Markov Games (TT-POMGs), introduced by Greenwald et al.~\cite{greenwald2017solving}, generalize Extensive-Form Games (EFGs) that model non-cooperative, sequential decision-making games of imperfect and/or incomplete information.
TT-POMG is a suitable formalism for decision-based attacks and defenses, with the added benefit that it can be transformed into an equivalent belief state MDP, significantly simplifying its solution.

Prior work has explored the competition underlying adversarial example generation in no-box and white-box settings~\cite{bose2020adversarial, gao2022achieving}.
We instead focus on decision-based, interactive environments, with unknown but stationary dynamics: all other agents present are considered part of the environment and therefore fixed in their behavior.
By folding the strategies of other agents into the transition probabilities and the initial probability distribution of the game, an optimal policy computed in the resulting MDP will correspond to the best-response strategy in the original TT-POMG.
The congruence between TT-POMGs and MDPs has both theoretical \textit{and} practical value for securing AI-based systems: once adversarial agents and their capabilities are identified through rigorous threat modeling, the best-response strategy in the simulated environment yields the optimal defense.

\begin{figure}
    \centering
    \includegraphics[width=0.48\textwidth]{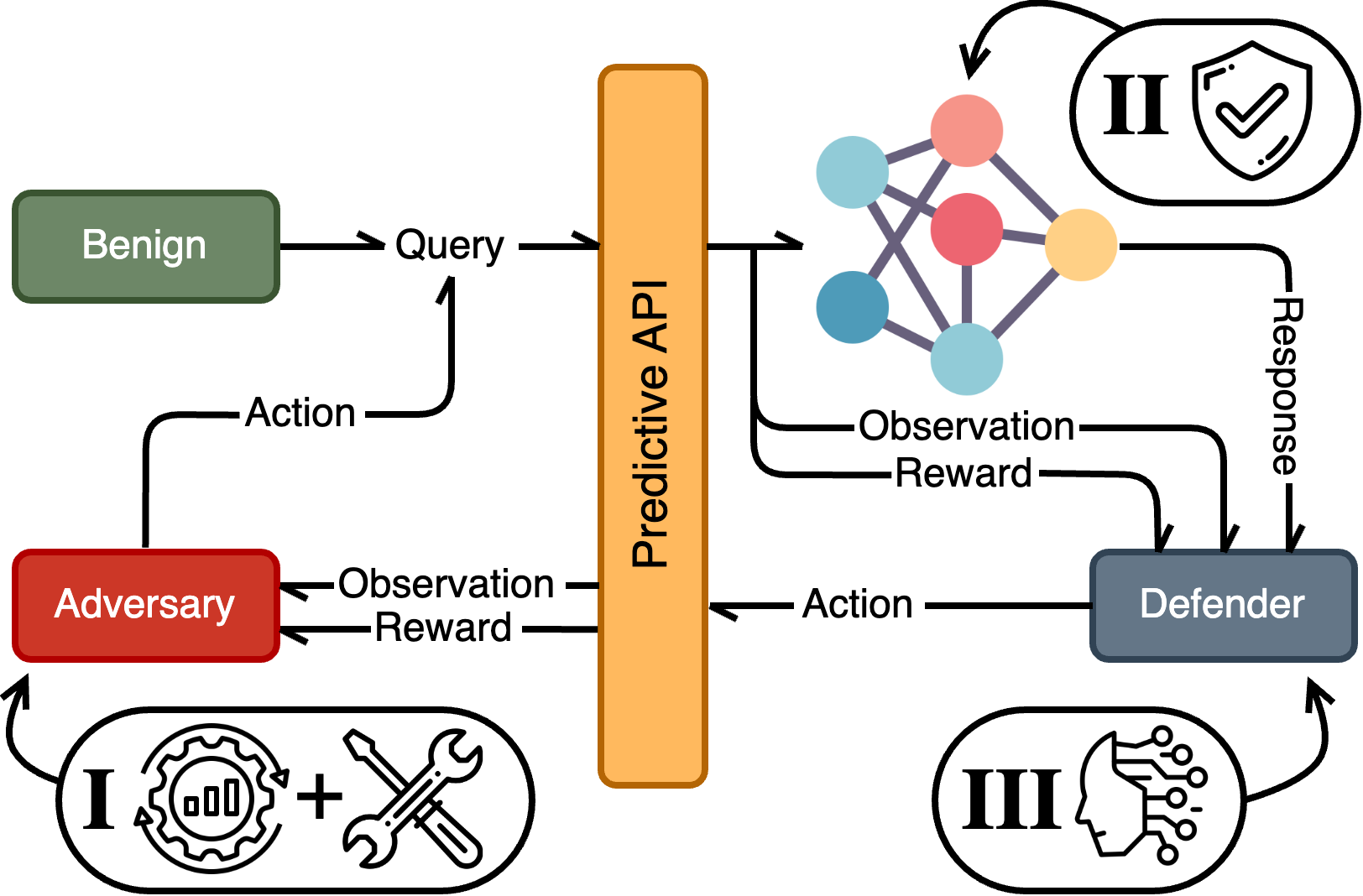}
    \caption{Schematic model of an AMG environment. Due to the inherent uncertainty of behavior at either side of the interface, it is a partially observable MDP, mirrored for each agent where one's decisions become the other's observations. (I) denotes an adaptive attacker (cf. Fig. \ref{fig:adaptivity}), (II) model hardening (passive defense), and (III) an active defense.}
    \label{fig:model}
\end{figure}

We describe the environment that encompasses adversarial attacks, adversarial defenses, and benign queries, as an Adversarial Markov Game (AMG) -- a special case of TT-POMG -- depicted in~\autoref{fig:model}.
Formally, we represent AMG as a tuple $\langle i, S, O, A, \tau, r, \gamma \rangle$
\begin{itemize}
    \item $i = \{\mathcal{D}, \mathcal{A}\}$ are the players, where $\mathcal{D}$ denotes the defender and $\mathcal{A}$ denotes the adversary. In our model, benign queries are modeled as moves by nature.
    \item $S$ is the full state space of the game, while $O = \{O^{\mathcal{D}}, O^{\mathcal{A}}\}$ are partial observations of the full state for each player.
    \item $A = \{A^{\mathcal{D}}, A^{\mathcal{A}}\}$ denotes the action set of each player.
    \item $\tau(s,a^{i},s')$ represents the transition probability to state $s' \in S$ after player $i$ chooses action $a^i$.
    \item $r = \{r^{\mathcal{D}}, r^{\mathcal{A}}\} : O^i \times A^i \rightarrow \mathbb{R}$ is the reward function where $r^i(s, a^i)$ is the reward of player $i$ if in state $s$ action $a^i$ is chosen.
    \item $\gamma^i \in [0,1)$ is the discount factor for player $i$.
\end{itemize}

The goal of each player $i$ is to determine a policy $\pi^{i}(A^i | O^i)$ that, given the policy of the other(s), maximizes their expected reward.
When a player employs a stationary policy, the AMG reduces to a belief-state MDP where the other interacts with a fixed environment.
The game is sequential and turn-taking, so each player $i$ chooses an action $a$ from their set of actions $A^i$ which subsequently influences the observations of others.

We have shown that an adaptive defense policy $\pi^{\mathcal{D}}_\phi$ is necessary to deter decision-based attacks, and that consequently the candidate generation policy $\pi^{\mathcal{A}}_\theta$ has to be also be adaptive.
As with plausible assumptions we cannot assume access to the exact state of the other agent, the states $O^{\mathcal{D}}$, $O^{\mathcal{A}}$ are partial observations of the complete state $\mathcal{S}$ of the full game.
For instance, when the competing agents (holding beliefs about each other) are human, they engage in recursive reasoning expressed as [I believe that [my opponent believes [that I believe...]]].
In the study of opponent modeling, considering other agent policies as a stationary part of the environment is equivalent to \textit{0th} level recursive reasoning: the agent models how the opponent behaves based on the observed history, but \emph{not} how the opponent \emph{would} act based on how the agent behaves~\cite{albrecht2018autonomous, wen2019probabilistic}.
In this work we consider more involved recursive reasoning out of scope, as AMGs can be solved by single-agent RL algorithms, and perform the empirical evaluation without building explicit models of opponent behavior.

%% file: threat.tex
\section{Threat Model}
\label{sec:threat}

The empirical study we conduct in Section~\ref{sec:evaluation} reflects diverse instantiations of the general theoretical framework introduced in Section~\ref{sec:approach}.
When working forward from the theoretical to the practical, concrete design choices have to be made when specifying the latter, choices that can have considerable influence on the results.
To elucidate our proposed robustness evaluation methodology, in this section we provide the concrete details on the threat model and the environment.

\textbf{[Threat Model].} Our AMG framework describes a two-player competitive game; while extensible to more players, in this work we assume that at a given moment only one attack takes place.
From the defensive perspective, incoming queries can be either benign or part of an attack.
An assumption that influences the effectiveness of stateful detection is that queries can be attributed to UIDs, e.g., an IP address or a user account.
However, adversaries can collude, create multiple accounts, use VPNs, or in fact accounts and IP addresses might not even be necessary to query the model.
To address this, we treat queries irrespective to their source.
This is a strictly more challenging setting for stateful defenses, where we operate solely on the content of queries and not on any other metadata, similar to~\cite{li2022blacklight}.
Unlike Blacklight however, instead of rejecting queries, something that in itself provides \textit{more} information to the adversary and thus facilitates evasion (cf. OARS \cite{feng2023stateful}, \autoref{tab:comparison}), we misdirect by returning the second highest probability class.
Furthermore, Gaussian noise is added to the benign queries to simulate a noisy channel and a shift in distribution, so that is not trivial for a defense to tell adversarial noise apart.
In summary, the black-box threat model we consider is delineated as follows:

\begin{itemize}[wide, labelindent=5pt, noitemsep, nolistsep]
    \item \textbf{Assets:} Trained and deployed model $\mathcal{M}$ with corresponding weights ${w}$.
    \item \textbf{Agents:} Adversary / Defender / Benign user.
    \item \textbf{Adversary Goal:} Generate minimal perturbation adversarial examples in as few queries as possible, while evading the defense.
    \item \textbf{Defender Goal:} Stop the adversary from generating adversarial examples, while preserving the correct functionality of the model $\mathcal{M}$ on benign users.
    \item \textbf{Adversary Knowledge:} The model $\mathcal{M}$ is known as the black-box function that transforms inputs $x \in [0,1]^d$ to outputs $c \in [m]$, $m$ being the number of classes. The weights ${w}$ and the closed-form expression of $\mathcal{M}$ are unknown, as unknown is if an active defense $\pi_\phi^{\mathcal{D}}$ is present or not.
    \item \textbf{Defender Knowledge:} The defender observes only the content of incoming queries, without knowing if they come from a benign user or the adversary.
    \item \textbf{Adversary Capabilities:} Adapt the parameters of the attack and of any evasive transformations; in essence, optimize the candidate generation policy $\pi^{\mathcal{A}}_\theta$.
    \item \textbf{Defender Capabilities:} For each query $x$, decide between answering truthfully with the actual prediction $C(x) = c_T$, or misdirect with the second highest probability class $c_S$.
\end{itemize}

\textbf{[Similarity].} Decision-based attacks typically follow a policy that generates successive queries: these exhibit degrees of similarity which can be quantified by an appropriate $L_p$ norm.
If that norm is computed on the original inputs however, an adversary can adapt by employing evasive transformations the model is invariant to and bypass the similarity detection (cf. \autoref{fig:adaptivity}).
To account for this capability, we train a Siamese network with contrastive loss in order to learn a latent space $\mathcal{L}(\cdot)$ where similar inputs are mapped close together, unaffected by added noise or transformations on the inputs.
For the stateful characterization of queries, we use two queues: one for the detected adversarial queries as determined by the defensive agent, and one for the benign and undetected ones.

\textbf{[Active defense].} Recall that decision-based attacks evaluate a Boolean-valued function to determine if the query is adversarial or not; a straightforward counter to this behavior is to misdirect by returning a decision different from the actual through a system of confinement.
When new query $x_t$ is received, a state is constructed based on $x_t$ and the queue $k_{-n}, k_{-n-1}, ..., k_{0}$ of known adversarial queries.
Based on this state, the defensive agent takes a single continuous action $\{\sigma \in \mathbb{R} \,|\, 0\leq \sigma \leq1\}$, with $\sigma$ being the radius of a hypersphere centered on the last known adversarial query $k_0$ in the latent space $\mathcal{L}$.
If $\|\mathcal{L}(x_t) - \mathcal{L}(k_0)\|_2 < \sigma$ the query is considered adversarial and is appended to the adversarial queue as the latest $k_0$.
This system of confinement is depicted in ~\autoref{fig:boundary}.

\textbf{[Adaptivity].} No evaluation in \ac{AML} is complete without considering adaptive adversaries; a notion we expand in this work, that is with the instruments to bypass the defense \textit{and} their optimal configuration.
As stateful defenses are so far similarity based, to bypass them intuition points towards input transformations the model is invariant to.
For a given query $x_t$ we want to compute a transformation $x_t' = T(x_t)$ so that $||x_t' - x_t||_2 \gg ||x_t - x_{t-1}||_2$ while $F(T(x_t)) \approx F(x_t)$.
Depending on magnitude and composition of transformations $T$, the identity $F(T(x_t)) = F(x_t)$ might not always hold.
As we also demonstrate in Section \ref{sec:evaluation}, $T$ interferes with the perturbations of the adversarial policy: the performance and evasiveness of an attack are thus in a natural trade-off.

At this point one should inquire what is the correct composition of transformations $T$ to apply.
When shall $T$ be applied, and how does it affect the attack fundamentals?
The transformations $T$ can be considered as a set of additional controls, and like attack parameters they themselves can be suboptimal out-of-the-box \cite{croce2020reliable}.
Thus the combined control of attack and evasion parameters is a \emph{prerequisite} to properly assess the strength of a defense.
Their trade-off illustrates why the twofold definition of adaptive is necessary in \ac{AML} evaluations: first to impart the tools to accomplish to the task through the definition of \textit{what} can be controlled, and then to find the precise optimal configuration and strategy of the attack.

\begin{figure}
    \centering
    \includegraphics[width=0.49\textwidth]{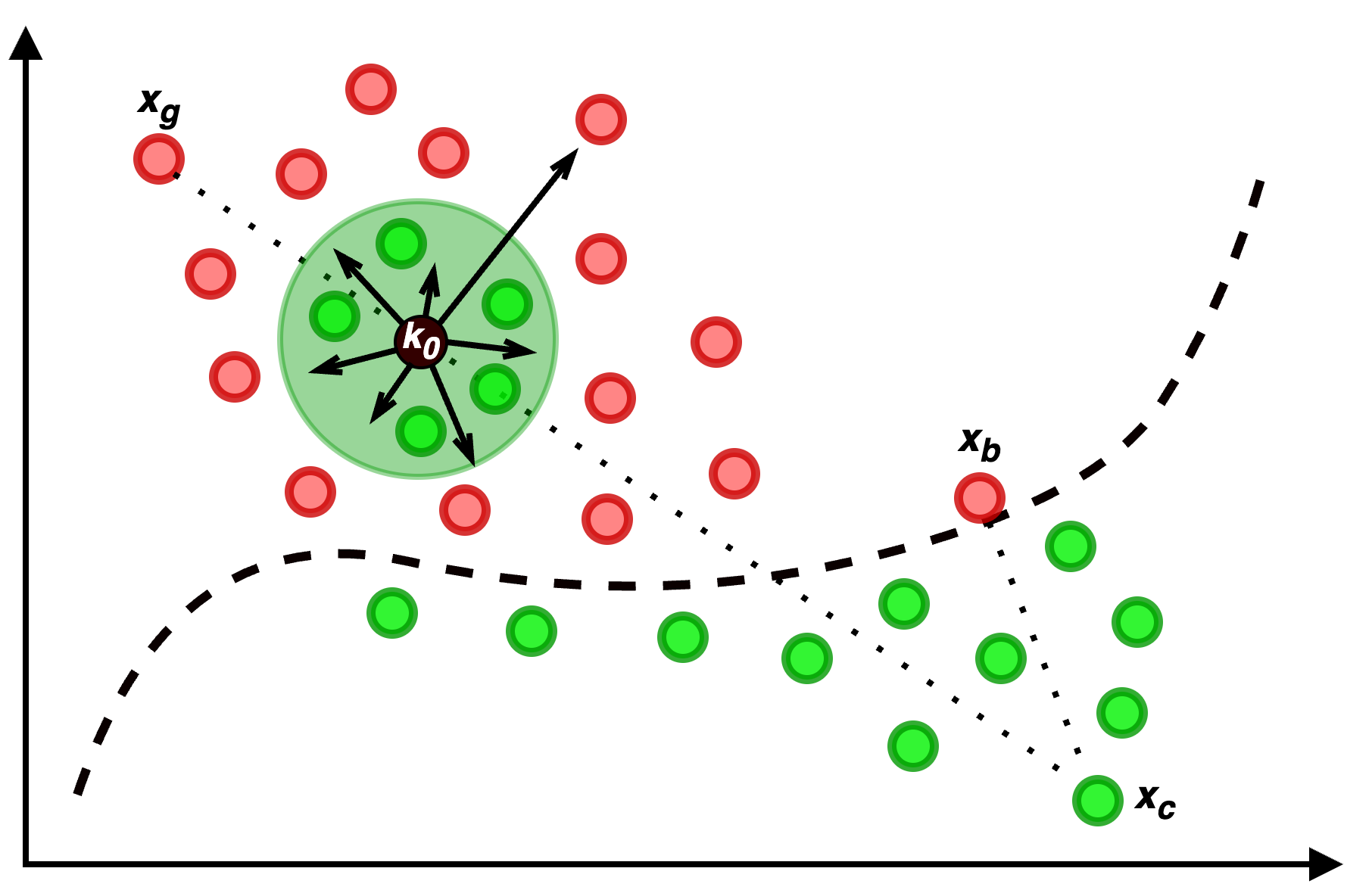}
    \caption{Misdirection in a hypothetical 2D decision boundary. The adaptive defense controls a single parameter, the hypersphere radius around $k_0$ (the last known adversarial); for queries $x_t$ that fall within this hypersphere the model responds with a non-adversarial decision. $x_g$ is the starting sample, $x_c$ the original, and $x_b$ the best possible adversarial.}
    \label{fig:boundary}
\end{figure}

\textbf{[Agents \& Environments].} Unlike common competitive games, in AMGs the two players have different action and state sets.
AMGs are also asymmetric in the playing cadence: while the defender plays every round, the adversary might wait one to several rounds; HSJA for example is controlled on the iteration rather on the query level.
Training is complicated further given that the experience upon which each agent learns arrives only \emph{after} the opponent moves.
We address these challenges by developing custom learning environments (with the OpenAI Gym and Stable-Baselines3 libraries) for asymmetrical agents, with delayed experience collection, and asynchronous training.

\textbf{[States \& Actions]}. 
For the definition and the rationale behind the states we use, we point the reader to Appendix~\ref{apx:rew}.
For actions, we control BAGS through 4 parameters: orthogonal step size, source step size, mask bias, and Perlin bias.
HSJA is controlled by 3: the gradient estimation radius, the number of estimation queries, and the jump step size.
All evaluations start from controlling these attack parameters \textit{only}; if the active defense proves impossible to defeat, we introduce additional knobs that control the magnitude and probability of transformations on the input, with the goal to evade detection \emph{while} preserving semantic content and hence the correct classification.
The range of transformations we experimented with as well as their magnitude and probability are listed in~\autoref{tbl:transforms}.
Finally, in both BAGS and HSJA the active defense consists of an 1-dimensional continuous action that controls the radius of confinement $a$, as depicted in Fig~\ref{fig:boundary}.

\textbf{[Rewards]}. Success in an \ac{RL} task relies heavily on \emph{how} it is rewarded.
Engineering an effective reward function is non-trivial and hides intricacies, as reward hacking and specification gaming are common phenomena and the learned behavior can vary~\cite{amodei2016concrete}.
For adversaries, the rewards we experimented with are variations on minimizing the distance to the original example -- with extra reward shaping based on the fundamental operation of each attack -- while defenders are rewarded or penalized for intercepting adversarial or benign queries respectively.
The rewards are described in Appendix~\ref{apx:rew}.

%% file: evaluation.tex
\section{Evaluation}
\label{sec:evaluation}
For evaluation, we define a range of scenarios intended to reflect all possible and realistic combinations between adversarial attacks and defenses, and their adaptive versions.
Concretely, the research questions we want to evaluate are: 1) Are active defenses a necessary complement to model hardening and to what extent? 2) Are attacks more threatening when adaptive, i.e., do they outperform their vanilla versions \textit{and} evade active detection? 3) If yes, to what extend active defenses recoup their performance by also turning adaptive?

\textbf{[Metrics]}. We employ \textbf{ASR} (Attack Success Rate) and $\mathbf{L_2}$ \textbf{norm} of the perturbation.
For the former we set a fixed threshold of 3 for consistency between experiments, while the latter is a more fine-grained metric well suited for comparing baseline attacks, defenses, and their adaptive versions, as it is not based on an (arbitrary) perceptual threshold that can yield widely varying results when moved.
The budgets we evaluate over are 1K, 2K and 5K queries.
As robustness and classification accuracy are typically in trade-off, the third metric of interest is the benign sample accuracy (\textbf{Clean Acc.}) that the original model and the active defense achieve together.

\subsection{Evaluation Setup}
Our goal is to learn offensive or defensive policies that are \emph{general}: they transfer to \textit{any} other evasion task.
Thus after training and validating the agents, the final performance is reported on a fixed hold-out set of 100 adversarial episodes where the starting and original samples are selected at random.
As is best practice in \ac{AML}, candidate samples are only those that are correctly classified by the model.
For each scenario we perform a limited hyperparameter and reward function exploration (max 30 trials), with the intention to root out poor combinations rather than exhaust the search space, described in more detail in Appendix~\ref{apx:hyper}.

The black-box attacks we render adaptive and evaluate are \textbf{BAGS} and \textbf{HSJA}, as they represent two fundamentally different approaches, are highly effective, \emph{and} have the highest evasion potential \cite{li2022blacklight}.
BAGS is a stochastic, search-based method where every query submitted is a new and potentially better adversarial example.
Contrastively, HSJA is deterministic and composed of 3 different stages where the queries are generated in an aggregated manner: the vast majority of them are not candidate adversarial examples but means of approximating the gradient at the decision boundary.

In training and evaluation, the adversarial game is played as follows: the adversary starts by submitting a query, then the defender responds either \emph{truthfully} (the actual model prediction) or by \emph{misdirecting} (the second highest probability class).
Then the environment decides with chance $p$ if the adversary moves next, otherwise a benign query is drawn.
In either case, it is the defender's turn; during testing they are also oblivious to the nature of the query and know only the content.
All experiments are performed with $p=0.5$; we also evaluate our trained defense when no attack is present ($p=0$) in Appendix~\ref{baserate}.

\begin{table*}[!ht]
\small
\centering
  \caption{\small{ASR and mean $L_2$ perturbation for 1K, 2K, and 5K queries for CIFAR-10, against normally and adversarially trained models. Clean Acc. reports the accuracy on benign queries of the base model plus any defenses present; in the first two scenarios (no active defense) the baseline clean accuracy is reported. Yellow scenarios denote the baseline attack performance, while green and red denote defensive and offensive scenarios respectively. The asterisk denotes where input transformations were used for evasion.}}
  \begin{tabular}{c|c|rrrr|rrrr|rr}
    \toprule
      \multirow{3}{*}{\parbox{1cm}{\centering Adv.\\Trained}} & \multirow{3}{*}{Scenario} & \multicolumn{10}{c}{\textbf{CIFAR-10 Gap: 20.01}} \\
      & &
      \multicolumn{4}{c}{BAGS} &
      \multicolumn{4}{c}{HSJA} &
      \multicolumn{2}{c}{Clean Acc.} \\
      \cline{3-12}
      & & {1K} & {2K} & {5K} & {ASR} & {1K} & {2K} & {5K} & {ASR} & {BAGS} & {HSJA}  \\
      \toprule
    \multirow{9}{*}{\xmark} & \raggedright\textbf{\textcolor{orange!70}{\phantom{*}0: VA-ND}} & 8.27 & 7.86 & 7.26 & \textcolor{t5!100}{\textbf{5\%}} & 3.42 & 1.43 & 0.41 & \textcolor{t100!100}{\textbf{100\%}} & 91.69 & 91.69 \\
    & \raggedright\textbf{\textcolor{purple!70}{\phantom{*}1: AA-ND}} & 1.26 & 0.71 & 0.49 & \textcolor{t100!100}{\textbf{100\%}} & 3.14 & 1.31 & 0.39 & \textcolor{t100!100}{\textbf{100\%}} & 91.69 & 91.69 \\
    & \raggedright\textbf{\textcolor{teal!70}{\phantom{*}2: VA-VD}} & 15.27 & 15.26 & 15.20 & \textcolor{t0!100}{\textbf{0\%}} & 11.14 & 10.81 & 10.33 & \textcolor{t7!100}{\textbf{7\%}} & 91.68 & 91.68 \\
    & \raggedright\textbf{\textcolor{purple!70}{\phantom{*}3: AA-VD}} & 2.63 & 2.03 & 1.77 & \textcolor{t93!100}{\textbf{93\%}} & 5.68 & 3.61 & 2.12 & \textcolor{t85!100}{\textbf{85\%}} & 91.69 & 91.69 \\
    & \raggedright\textbf{\textcolor{teal!70}{\phantom{*}4: VA-AD}} & 20.01 & 20.01 & 20.00 & \textcolor{t0!100}{\textbf{0\%}} & 17.17 & 16.35 & 15.56 & \textcolor{t0!100}{\textbf{0\%}} & 91.60 & 91.50 \\
    & \raggedright\textbf{\textcolor{purple!70}{*5: AA-TD}} & 6.28 & 5.45 & 4.52 & \textcolor{t30!100}{\textbf{30\%}} & 13.19 & 11.82 & 10.69 & \textcolor{t2!100}{\textbf{2\%}} & 91.52 & 91.46 \\
    & \raggedright\textbf{\textcolor{teal!70}{*6: TA-AD}} & 19.52 & 19.40 & 18.95 & \textcolor{t0!100}{\textbf{0\%}} & 16.48 & 16.13 & 15.69 & \textcolor{t0!100}{\textbf{0\%}} & 91.38 & 91.62 \\
    & \raggedright\textbf{\textcolor{purple!70}{*7: AA-AD}} & 9.95 & 9.80 & 9.80 & \textcolor{t5!100}{\textbf{5\%}} & 10.30 & 9.04 & 7.55 & \textcolor{t23!100}{\textbf{23\%}} & 91.66 & 91.55 \\
    & \raggedright\textbf{\textcolor{teal!70}{*8: AA-AD}} & 19.85 & 19.85 & 19.85 & \textcolor{t0!100}{\textbf{0\%}} & 14.46 & 13.93 & 13.08 & \textcolor{t1!100}{\textbf{1\%}} & 91.69 & 91.37 \\
    \midrule
    \multirow{9}{*}{\cmark} & \raggedright\textbf{\textcolor{orange!70}{\phantom{*}0: VA-ND}} & 8.72 & 8.42 & 7.94 & \textcolor{t4!100}{\textbf{4\%}} & 3.73 & 1.74 & 0.75 & \textcolor{t100!100}{\textbf{100\%}} & 87.76 & 87.76 \\
    & \raggedright\textbf{\textcolor{purple!70}{\phantom{*}1: AA-ND}} & 1.74 & 1.13 & 0.79 & \textcolor{t100!100}{\textbf{100\%}} & 3.64 & 1.77 & 0.73 & \textcolor{t100!100}{\textbf{100\%}} & 87.76 & 87.76 \\
    & \raggedright\textbf{\textcolor{teal!70}{\phantom{*}2: VA-VD}} & 15.42 & 15.35 & 15.20 & \textcolor{t0!100}{\textbf{0\%}} & 11.10 & 10.73 & 10.38 & \textcolor{t4!100}{\textbf{4\%}} & 87.72 & 87.73 \\
    & \raggedright\textbf{\textcolor{purple!70}{\phantom{*}3: AA-VD}} & 2.82 & 2.26 & 2.06 & \textcolor{t81!100}{\textbf{81\%}} & 5.66 & 3.36 & 1.94 & \textcolor{t86!100}{\textbf{86\%}} & 87.74 & 87.74 \\
    & \raggedright\textbf{\textcolor{teal!70}{\phantom{*}4: VA-AD}} & 20.01 & 20.01 & 20.00 & \textcolor{t0!100}{\textbf{0\%}} & 17.06 & 16.40 & 15.81 & \textcolor{t0!100}{\textbf{0\%}} & 87.66 & 87.66 \\
    & \raggedright\textbf{\textcolor{purple!70}{*5: AA-TD}} & 8.48 & 7.68 & 6.82 & \textcolor{t9!100}{\textbf{9\%}} & 13.59 & 12.65 & 11.39 & \textcolor{t1!100}{\textbf{1\%}} & 87.58 & 87.52 \\
    & \raggedright\textbf{\textcolor{teal!70}{*6: TA-AD}} & 19.58 & 19.40 & 18.95 & \textcolor{t0!100}{\textbf{0\%}} & 16.60 & 16.26 & 15.99 & \textcolor{t0!100}{\textbf{0\%}} & 87.50 & 87.68 \\
    & \raggedright\textbf{\textcolor{purple!70}{*7: AA-AD}} & 10.43 & 10.24 & 10.17 & \textcolor{t1!100}{\textbf{1\%}} & 10.21 & 9.22 & 7.82 & \textcolor{t12!100}{\textbf{12\%}} & 87.73 & 87.61 \\
    & \raggedright\textbf{\textcolor{teal!70}{*8: AA-AD}} & 19.86 & 19.86 & 19.86 & \textcolor{t0!100}{\textbf{0\%}} & 15.71 & 15.35 & 14.30 & \textcolor{t1!100}{\textbf{1\%}} & 87.67 & 87.40 \\
    \bottomrule
  \end{tabular}
  \label{tab:result2}
\end{table*}

The scenarios for all possible combinations of (non-) adaptive attacks and defenses are repeated over two datasets -- CIFAR10 and MNIST -- and over two models with the same architecture but different training regimes: with and without adversarial training.
As the transition from single to multi-agent \ac{RL} introduces non-stationarity, we approach the AMG as a belief-state MDP (relaxing the requirement of knowing the exact opponent policies), and use PPO \cite{schulman2017proximal} agents to learn optimal policies that will also constitute best responses for the full game~\cite{wen2019probabilistic}.
Note that learning independently of other agency breaks the theoretical guarantees of convergence~\cite{tuyls2012multiagent}, eg. in scenarios 7 \& 8 where both agents learn simultaneously.
Coloring denotes the learning/evaluated agent in each scenario, with their complete list being as follows:

\begin{enumerate}[leftmargin=*]
\setcounter{enumi}{-1} 
    \setlength\itemsep{0.1em}
    \item \textbf{VA-ND} -- \textcolor{orange}{Vanilla Attack} / No Defense: Baseline performance of attacks (BAGS \& HSJA) out-of-the-box, without any active defense.
    \item \textbf{AA-ND} -- \textcolor{purple}{Adaptive Attack} / No Defense: How much more optimal is the adaptive version compared to the baseline attack.
    \item \textbf{VA-VD} -- Vanilla Attack / \textcolor{teal}{Vanilla Defense}: The performance of our active defense, the non-adaptive version that has an empirically defined detection threshold.
    \item \textbf{AA-VD} -- \textcolor{purple}{Adaptive Attack} / Vanilla Defense: Similar to scenario (2), but now the attack is adaptive.
    \item \textbf{VA-AD} -- Vanilla Attack / \textcolor{teal}{Adaptive Defense}: The first scenario where the active defense is also adaptive, against the baseline adversary.
    \item \textbf{AA-TD} -- \textcolor{purple}{Adaptive Attack} / Trained Defense: After the adaptive defense is optimized, its policy is fixed and an adaptive attack is trained against it.
    \item \textbf{TA-AD} -- Trained Attack / \textcolor{teal}{Adaptive Defense}: The best policy found in the previous scenario is fixed and an adaptive defense is trained against it.
    \item \textbf{AA-AD} -- \textcolor{purple}{Adaptive Attack} / \textcolor{teal}{Adaptive Defense}: The first scenario where both agents learn simultaneously, making the environment non-stationary. In practice, the convergence will vary and depend on the chosen hyperparameters and rewards. Here we report the best-case for the attack.
    \item \textbf{AA-AD} -- \textcolor{purple}{Adaptive Attack} / \textcolor{teal}{Adaptive Defense}: The exact setup as scenario 7, but the best-case for the defense is reported instead.
\end{enumerate}


In each successive scenario, we evaluate using the most successful past policy, following standard practice in Markov Games: the worst-case opponent policy is fixed, and a best response to it is learned~\cite{littman1994markov, timbers2022approximate}.
Fixing other policies when computing a best response stabilizes learning in multi-agent environments, as it simplifies the problem to a single-agent setting -- one that, as discussed in Section~\ref{sec:AMG}, can be solved with standard \ac{RL}.

\textbf{Comparison to SotA.} In Scenarios 0-8 we evaluate all possible combinations between (adaptive) attack and defenses.
As a baseline to compare to, we additionally evaluate our approach to the state-of-the-art stateful defenses and adaptive attacks, that is Blacklight \cite{li2022blacklight} and OARS \cite{feng2023stateful} respectively.
We implement both Blacklight and OARS in our interactive environments by using their publicly available code and parameters.
As our environments do not return a rejection signal and to make a fair comparison, for OARS rejection coincides with a non-adversarial decision.
We thus define 5 further scenarios:

\begin{enumerate}[leftmargin=*]
\setcounter{enumi}{8} 
    \setlength\itemsep{0.1em}
    \item \textbf{VA-BD} -- \textcolor{orange}{Vanilla Attack} / Blacklight Defense: Baseline performance of the attacks against Blacklight.
    \item \textbf{OA-BD} -- \textcolor{purple}{OARS Attack} / Blacklight Defense: OARS against Blacklight.
    \item \textbf{AA-BD} -- \textcolor{purple}{Adaptive Attack} / Blacklight Defense: Our adaptive attack against Blacklight.
    \item \textbf{OA-TD} -- OARS Attack / \textcolor{teal}{Trained Defense}: OARS against our trained defense from Scenario 6.
    \item \textbf{OA-AD} -- OARS Attack / \textcolor{teal}{Adaptive Defense}: Our adaptive defense retrained against OARS.
\end{enumerate}

Our experiments were run on multiple machines, yet to give an idea for the time complexity of our defense, on an Intel i7-7700 CPU one forward pass in CIFAR -- that is one response to one query -- takes $8 \pm 1.4$ ms for $\sim$700 MFLOPs.

\subsection{Results}
For consistency and comparability between evaluations, all results are from the \emph{same} 100 test episodes.
The \textbf{gap} value denotes the $L_2$ perturbation that initially separates the starting and the original samples, averaged over the 100 episodes. 
By testing the trained agents on budgets higher than 5K we discovered that the trend in reducing $L_2$ holds; to make the agent training tractable and the evaluation wider however, we limit the maximum query budget per adversarial episode to 5K.
~\autoref{tab:result2} \& \autoref{tab:result3} report the results for CIFAR10, while ~\autoref{tab:result1} reports MNIST.
The closer examination of the empirical results help us extract and highlight several important insights, practical observations, and general implications for the broader \ac{AML} field:

\begin{table*}[!ht]
\small
\centering
  \caption{\small{ASR and mean $L_2$ perturbation for CIFAR-10, comparing our Adaptive Attack (AA) and Adaptive Defense (AD) to Blacklight (BD) and OARS (OA).}}
  \begin{tabular}{c|c|rrrr|rrrr|rr}
    \toprule
      \multirow{3}{*}{\parbox{1cm}{\centering Adv.\\Trained}} & \multirow{3}{*}{Scenario} & \multicolumn{10}{c}{\textbf{CIFAR-10 Gap: 20.01}} \\
      & &
      \multicolumn{4}{c}{BAGS} &
      \multicolumn{4}{c}{HSJA} &
      \multicolumn{2}{c}{Clean Acc.} \\
      \cline{3-12}
      & & {1K} & {2K} & {5K} & {ASR} & {1K} & {2K} & {5K} & {ASR} & {BAGS} & {HSJA}  \\
      \toprule
    \multirow{5}{*}{\xmark} & \raggedright\textbf{\textcolor{orange!70}{\phantom{*}9: VA-BD}} & 9.55 & 9.32 & 9.17 & \textcolor{t0!100}{\textbf{0\%}} & 8.41 & 8.19 & 7.80 & \textcolor{t15!100}{\textbf{15\%}} & 91.71 & 91.71 \\
    & \raggedright\textbf{\textcolor{purple!70}{10: OA-BD}} & 9.46 & 9.46 & 9.46 & \textcolor{t1!100}{\textbf{1\%}} & 6.54 & 5.83 & 4.67 & \textcolor{t50!100}{\textbf{50\%}} & 91.71 & 91.71 \\
    & \raggedright\textbf{\textcolor{purple!70}{11: AA-BD}} & 2.26 & 1.39 & 1.32 & \textcolor{t98!100}{\textbf{98\%}} & 4.55 & 3.08 & 2.44 & \textcolor{t78!100}{\textbf{78\%}} & 91.71 & 91.71 \\
    & \raggedright\textbf{\textcolor{teal!70}{12: OA-TD}} & 20.01 & 20.01 & 20.01 & \textcolor{t0!100}{\textbf{0\%}} & 7.07 & 6.38 & 5.53 & \textcolor{t50!100}{\textbf{50\%}} & 91.61 & 91.59 \\
    & \raggedright\textbf{\textcolor{teal!70}{13: OA-AD}} & 20.01 & 20.01 & 20.01 & \textcolor{t0!100}{\textbf{0\%}} & 11.03 & 11.00 & 10.95 & \textcolor{t5!100}{\textbf{5\%}} & 91.61 & 91.69 \\
    \midrule
    \multirow{5}{*}{\cmark} & \raggedright\textbf{\textcolor{orange!70}{\phantom{*}9: VA-BD}} & 9.75 & 9.56 & 9.46 & \textcolor{t0!100}{\textbf{0\%}} & 8.67 & 8.50 & 8.28 & \textcolor{t7!100}{\textbf{7\%}} & 87.76 & 87.76 \\
    & \raggedright\textbf{\textcolor{purple!70}{10: OA-BD}} & 9.79 & 9.79 & 9.79 & \textcolor{t1!100}{\textbf{1\%}} & 5.77 & 4.53 & 3.26 & \textcolor{t72!100}{\textbf{72\%}} & 87.76 & 87.76 \\
    & \raggedright\textbf{\textcolor{purple!70}{11: AA-BD}} & 5.59 & 4.04 & 2.55 & \textcolor{t79!100}{\textbf{79\%}} & 5.59 & 4.04 & 2.55 & \textcolor{t79!100}{\textbf{79\%}} & 87.76 & 87.76 \\
    & \raggedright\textbf{\textcolor{teal!70}{12: OA-TD}} & 20.01 & 20.01 & 20.01 & \textcolor{t0!100}{\textbf{0\%}} & 6.44 & 5.49 & 4.38 & \textcolor{t65!100}{\textbf{65\%}} & 87.66 & 87.64 \\
    & \raggedright\textbf{\textcolor{teal!70}{13: OA-AD}} & 20.01 & 20.01 & 20.01 & \textcolor{t0!100}{\textbf{0\%}} & 11.31 & 11.12 & 10.97 & \textcolor{t7!100}{\textbf{7\%}} & 87.66 & 87.74 \\
    \bottomrule
  \end{tabular}
  \label{tab:result3}
\end{table*}

\begin{itemize}[wide, labelindent=5pt, noitemsep, nolistsep]
    \setlength\itemsep{0.35em}
    \item The initial performance of an attack can be misleading: out-of-the-box HSJA appears to be the better attack, but it is often outperformed by adaptive BAGS, especially in CIFAR and against active defenses.
    \item The performance of both attacks deteriorates considerably against active defenses (VD), however the defenses reach their full potential only when \emph{also adaptive} (AD).
    \item Our adaptive defense (AD) outperforms both Blacklight (BD) and non-adaptive (VD) defenses, also when transferred (cf. Sc.12). Compared to Blacklight, it reduces ASR by $\sim$90\% (in HSJA) when trained against OARS, while it offers similar protection when transferred from another attack.
    \item Overall, against the strongest attacks and in the worst case (Sc.7) for it, our defense contains the ASR in the range of $1-36\%$.
    \item Our adaptive attack (AA) outperforms OARS (Sc.10 \& 11) and vanilla attacks (VA) by a wide margin, \textit{without} access to rejection sampling, and irrespective of the defense it faces; the one exception is our adaptive defense (AD), against which it has very limited success.
    \item The advantage of adaptive attacks is more pronounced against active defenses, where they significantly outperform non-adaptive versions, cf. Sc.0\textrightarrow1, Sc.2\textrightarrow3.
    \item When comparing the upper and lower halves of each table, we can observe that adversarial training adds a limited amount of robustness; otherwise, \emph{the practical effect of adversarial training is a tax on the attacker}, forcing them to expend more queries for the same perturbation or having higher perturbation for the same query budget.
    \item Evasive transformations interfere with the attack policy, as illustrated by the difference between BAGS and HSJA in Sc.5. For an attack to reach its full potential, these two should be adaptively controlled together.
    \item In Sc.1 to 8, agents train against the best opponent policy as previously discovered, and ASR oscillates since following a fixed policy enables the learning of an optimal counter to it. Over successive adaptations, ASR eventually plateaus, indicating an equilibrium for each specific dataset and attack combination. This is illustrated in \autoref{fig:plot}.
    \item The first time active defenses resisted adaptive attacks was in Sc.5 of CIFAR; we employed evasive transformations from then onward.
    \item Different attack fundamentals respond differently to active defenses; the gradient estimation stage of HSJA has a disadvantage against similarity detection, while the jump and binary stages have an advantage.
    \item For HSJA, engineering a state the adaptive defense could learn on, merely by leveraging our knowledge of the attack and its geometric functioning, proved impossible. What did prove effective however, was pure computation\footnote{This is reminiscent of Sutton's Bitter Lesson~\cite{sutton2019bitter}, the observation that progress in AI is often driven by gains in computation rather than problem-specific expert knowledge.}: we used Contrastive Learning~\cite{hadsell2006dimensionality} to learn an embedding from raw queries, then used as state that transfers exceedingly well to other attacks like BAGS.
\end{itemize}

\begin{table*}
\small
\centering
  \caption{\small{ASR and mean $L_2$ perturbation for 1K, 2K, and 5K queries and accuracy on clean data for MNIST. The evaluation scenarios are identical to Tables~\ref{tab:result2} and~\ref{tab:result3}.}}
  \begin{tabular}{c|c|rrrr|rrrr|rr}
    \toprule
      \multirow{3}{*}{\parbox{1cm}{\centering Adv.\\Trained}} & \multirow{3}{*}{Scenario} &
      \multicolumn{8}{c}{\textbf{MNIST Gap = 10.62}} \\
      & &
      \multicolumn{4}{c}{BAGS} &
      \multicolumn{4}{c}{HSJA} &
      \multicolumn{2}{c}{Clean Acc.} \\
      \cline{3-12}
      & & {1K} & {2K} & {5K} & {ASR} & {1K} & {2K} & {5K} & {ASR} & {BAGS} & {HSJA}  \\
      \toprule
    \multirow{14}{*}{\xmark} & \raggedright\textbf{\textcolor{orange!70}{\phantom{*}0: VA-ND}} & 5.30 &  5.28 & 5.26 & \textcolor{t3!100}{\textbf{3\%}} & 3.59 & 3.07 & 2.61 & \textcolor{t73!100}{\textbf{73\%}} & 99.37 & 99.37\\
    & \raggedright\textbf{\textcolor{purple!70}{\phantom{*}1: AA-ND}} & 2.74 & 2.57 & 2.47 & \textcolor{t78!100}{\textbf{78\%}} & 3.61 & 3.09 & 2.60 & \textcolor{t74!100}{\textbf{74\%}} & 99.37 & 99.37\\
    & \raggedright\textbf{\textcolor{teal!70}{\phantom{*}2: VA-VD}} & 7.44 & 6.66 & 5.63 & \textcolor{t22!100}{\textbf{22\%}} & 5.82 & 5.78 & 5.73 & \textcolor{t2!100}{\textbf{2\%}} & 99.34 & 99.20\\
    & \raggedright\textbf{\textcolor{purple!70}{\phantom{*}3: AA-VD}} & 3.79 & 3.66 & 3.44 & \textcolor{t29!100}{\textbf{29\%}} & 3.54 & 3.09 & 2.77 & \textcolor{t61!100}{\textbf{61\%}} & 99.37 & 99.31\\
    & \raggedright\textbf{\textcolor{teal!70}{\phantom{*}4: VA-AD}} & 10.57 & 10.57 & 10.57 & \textcolor{t0!100}{\textbf{0\%}} & 10.05 & 10.05 & 10.05 & \textcolor{t0!100}{\textbf{0\%}} & 99.31 & 99.30 \\
    & \raggedright\textbf{\textcolor{purple!70}{\phantom{*}5: AA-TD}} &  3.57 &  3.29 &  3.14 & \textcolor{t39!100}{\textbf{39\%}} & 5.00 &  3.97 &  3.38 & \textcolor{t36!100}{\textbf{36\%}} & 99.32 & 98.84 \\
    & \raggedright\textbf{\textcolor{teal!70}{\phantom{*}6: TA-AD}} & 10.62 & 10.62 & 10.62 & \textcolor{t0!100}{\textbf{0\%}} & 10.23 & 10.23 & 10.18 & \textcolor{t0!100}{\textbf{0\%}} & 99.28 & 99.34\\
    & \raggedright\textbf{\textcolor{purple!70}{\phantom{*}7: AA-AD}} &  4.89 &  4.89 &  4.86 & \textcolor{t8!100}{\textbf{8\%}} & 5.06 & 4.76 &  4.38 & \textcolor{t36!100}{\textbf{36\%}} & 99.31 & 99.35\\
    & \raggedright\textbf{\textcolor{teal!70}{\phantom{*}8: AA-AD}} & 10.62 & 10.62 & 10.62 & \textcolor{t0!100}{\textbf{0\%}} &  10.21 &  10.21 &  10.21 & \textcolor{t0!100}{\textbf{0\%}} & 99.32 & 99.23\\
    \cline{2-12}
    & \raggedright\textbf{\textcolor{orange!70}{\phantom{*}9: VA-BD}} & 10.62 & 10.62 & 10.62 & \textcolor{t0!100}{\textbf{0\%}} &  5.65 & 5.65 & 5.65 & \textcolor{t2!100}{\textbf{2\%}} & 99.37 & 99.37 \\
    & \raggedright\textbf{\textcolor{purple!70}{10: OA-BD}} & 10.62 & 10.62 & 10.62 & \textcolor{t0!100}{\textbf{0\%}} & 4.53 & 4.00 & 3.15 & \textcolor{t46!100}{\textbf{46\%}} & 99.37 & 99.37 \\
    & \raggedright\textbf{\textcolor{purple!70}{11: AA-BD}} & 3.83 & 3.69 & 3.60 & \textcolor{t17!100}{\textbf{17\%}} & 4.18 & 3.66 & 3.19 & \textcolor{t52!100}{\textbf{52\%}} & 99.37 & 99.37 \\
    & \raggedright\textbf{\textcolor{teal!70}{12: OA-TD}} & 10.62 & 10.62 & 10.62 & \textcolor{t0!100}{\textbf{0\%}} & 10.21 & 10.20 & 10.20 & \textcolor{t0!100}{\textbf{0\%}} & 99.22 & 99.28 \\
    & \raggedright\textbf{\textcolor{teal!70}{13: OA-AD}} & 10.62 & 10.62 & 10.62 & \textcolor{t0!100}{\textbf{0\%}} & 10.21 & 10.20 & 10.20 & \textcolor{t0!100}{\textbf{0\%}} & 99.32 & 99.28 \\
    \midrule
    \multirow{14}{*}{\cmark} & \raggedright\textbf{\textcolor{orange!70}{\phantom{*}0: VA-ND}} &  5.26 &  5.25 &  5.24 & \textcolor{t2!100}{\textbf{2\%}} &  4.61 &  4.04 &  3.41 & \textcolor{t30!100}{\textbf{30\%}} & 99.15 & 99.15\\
    & \raggedright\textbf{\textcolor{purple!70}{\phantom{*}1: AA-ND}} &  3.28 &  3.08 &  2.96 & \textcolor{t51!100}{\textbf{51\%}} &  4.59 &  3.97 &  3.35 & \textcolor{t34!100}{\textbf{34\%}} & 99.15 & 99.15\\
    & \raggedright\textbf{\textcolor{teal!70}{\phantom{*}2: VA-VD}} &  7.70 &  6.86 &  5.86 & \textcolor{t17!100}{\textbf{17\%}} &  5.81 &  5.78 &  5.76 & \textcolor{t2!100}{\textbf{2\%}} & 99.14 & 99.12\\
    & \raggedright\textbf{\textcolor{purple!70}{\phantom{*}3 AA-VD}} &  4.18 &  4.08 &  3.86 & \textcolor{t22!100}{\textbf{22\%}} &  4.63 &  4.27 &  3.86 & \textcolor{t25!100}{\textbf{25\%}} & 99.13 & 99.15 \\
    & \raggedright\textbf{\textcolor{teal!70}{\phantom{*}4: VA-AD}} & 10.55 & 10.55 & 10.55 & \textcolor{t0!100}{\textbf{0\%}} & 10.02 & 10.02 & 10.02 & \textcolor{t0!100}{\textbf{0\%}} & 99.09 & 99.08\\
    & \raggedright\textbf{\textcolor{purple!70}{\phantom{*}5: AA-TD}} &  4.04 &  3.74 &  3.54 & \textcolor{t27!100}{\textbf{27\%}} &  5.82 &  5.09 &  4.26 & \textcolor{t16!100}{\textbf{16\%}} & 99.11 & 98.78 \\
    & \raggedright\textbf{\textcolor{teal!70}{\phantom{*}6: TA-AD}} & 10.62 & 10.62 & 10.62 & \textcolor{t0!100}{\textbf{0\%}} & 10.20 & 10.20 & 10.20 & \textcolor{t0!100}{\textbf{0\%}} & 99.06 & 99.06\\
    & \raggedright\textbf{\textcolor{purple!70}{\phantom{*}7: AA-AD}} &  5.59 &  5.56 &  5.56 & \textcolor{t5!100}{\textbf{5\%}} &  5.47 &  5.16 &  4.99 & \textcolor{t14!100}{\textbf{14\%}} & 99.09 & 99.13\\
    & \raggedright\textbf{\textcolor{teal!70}{\phantom{*}8: AA-AD}} & 10.62 & 10.62 & 10.62 & \textcolor{t0!100}{\textbf{0\%}} & 10.12 & 10.12 & 10.12 & \textcolor{t0!100}{\textbf{0\%}} & 99.10 & 99.01 \\
    \cline{2-12}
    & \raggedright\textbf{\textcolor{orange!70}{\phantom{*}9: VA-BD}} & 10.62 & 10.62 & 10.62 & \textcolor{t0!100}{\textbf{0\%}} & 5.65 & 5.64 & 5.64 & \textcolor{t1!100}{\textbf{1\%}} & 99.15 & 99.15\\
    & \raggedright\textbf{\textcolor{purple!70}{10: OA-BD}} & 10.62 & 10.62 & 10.62 & \textcolor{t0!100}{\textbf{0\%}} & 5.18 & 4.80 & 4.11 & \textcolor{t17!100}{\textbf{17\%}} & 99.15 & 99.15 \\
    & \raggedright\textbf{\textcolor{purple!70}{11: AA-BD}} & 4.31 & 4.07 & 3.96 & \textcolor{t13!100}{\textbf{13\%}} & 5.04 & 4.65 & 4.20 & \textcolor{t19!100}{\textbf{19\%}} & 99.15 & 99.15 \\
    & \raggedright\textbf{\textcolor{teal!70}{12: OA-TD}} & 10.62 & 10.62 & 10.62 & \textcolor{t0!100}{\textbf{0\%}} & 10.26 & 10.26 & 10.26 & \textcolor{t0!100}{\textbf{0\%}} & 99.00 & 99.06\\
    & \raggedright\textbf{\textcolor{teal!70}{13: OA-AD}} & 10.62 & 10.62 & 10.62 & \textcolor{t0!100}{\textbf{0\%}} & 10.26 & 10.26 & 10.26 & \textcolor{t0!100}{\textbf{0\%}} & 99.10 & 99.06 \\
    \bottomrule
  \end{tabular}
  \label{tab:result1}
\end{table*}

%% file: discussion.tex
\section{Discussion}
\label{sec:discussion}

Our work has several implications for performing robust inference in the real-world.
While adversarial training remains the most reliable defense, the empirical robustness it imparts will vary and even be insufficient.
We note that this robustness is against \emph{all} adversarial examples under the same $L_p$-norm; active defenses protect only against querying attacks, but as they do transfer between attacks (cf. Scen. 12) they can be used jointly as complementary approaches.
We demonstrated how AI-enabled systems are susceptible to adaptive adversaries that \emph{devise} new evasive techniques and \emph{control them jointly} with other attack parameters.
This has been achieved in the \emph{fully black-box} case and \emph{against active defenses}.
Notably, the level of threat that adaptive adversaries pose against such systems is considerable, as it is straightforward to generalize Proposition~\ref{th:epg} to any other domain or modality.
This rekindles the proverbial arms race, where as a consequence defenses should also be equally capable and adaptive.

\textbf{Limitations}.
To keep the amount of evaluations practical, we narrowed the scope to targeted attacks and to $L_2$ as the more suitable norm for visual similarity.
Targeted attacks are strictly more difficult to perform than untargeted, while for binary classification targeted and untargeted coincide; our framework, however, can accommodate any adversarial goal or metric.
Another simplifying assumption we make is that only one attack can take place at a time; however, the queuing technique we use for incoming queries is readily extensible to handle concurrent attacks.
While we demonstrate that our active defense does transfer between attacks, another possibility to explore is training the defense on queries from different kinds of attacks.
Finally, in our evaluation we focus on a wide range of adaptive and non-adaptive scenarios where agents learn and adapt interactively, thus limiting the number of datasets we experiment with; we back our empirical study however with an extensive theoretical analysis that supports the generality of our findings independent of context.

\textbf{Future Work}.
The AMG framework we introduce is general by design and can accommodate the learning of optimal offensive and defensive policies in any domain of interest beyond image classification.
A promising path for future research is the extension of our adaptive attacks and defenses to other domains and modalities, for instance malware, bot, and network intrusion detection.
This is specifically because our approach circumvents the main obstacle of mapping gradient-based perturbations to feasible objects (eg. binaries) and instead can function directly in the problem space~\cite{pierazzi2020intriguing}.
Another compelling and formidable challenge is automating the adaptive evaluations in \ac{AML}, that is adapting beyond a specification by inventing tools to bypass defenses and thus imparting controllability to adversarial tasks.
Finally, in our work we considered opponent agency as part of the environment; other domains, like malware detection, might benefit from explicit opponent modeling.

\section{Conclusion}

With adaptive, decision-based attacks becoming more pervasive in multiple domains, every AI-based system that exposes a queryable interface is inherently vulnerable.
To aggravate matters, this vulnerability cannot be mitigated by employing model hardening approaches like adversarial training alone.
To fully defend in the presence of such attacks, active \emph{and} adaptive defenses are necessary, and we demonstrate how optimal defensive policies can be learned.
However, the existence of such defenses elicits in turn adaptive attacks which are able to recover part of their original performance.

We perform a theoretical and empirical investigation of decision-based attacks and stateful defenses under a unified framework we name ``Adversarial Markov Games'' (AMG).
In self-adaptive, we introduce a novel twofold definition of adaptive: both inventing new techniques to outmaneuver opponents \textit{and} adapting one's operating policy with respect to other agency in the environment.
Furthermore, through our theoretical analysis we demonstrate how any combination of adversarial goals, be it performance, stealthiness~\cite{debenedetti2024evading}, or disruption, can be optimized in a gradient based manner, even in the \emph{complete} black-box case and in \emph{any} domain.
As new attacks and defenses constantly emerge and are surpassed, our proposed methodology is generally applicable as it turns any such approaches in the current arms-race self-adaptive, thus ensuring accurate and robust assessment of their performance.

The AMG framework we introduce helps us reason on and properly assess the vulnerabilities of AI-based systems, disentangling the inherently complex and non-stationary task of learning in the presence of competing agency.
By modeling the latter as part of the environment, we can simplify this task by computing a best response to the observed behavior.
This has a significant consequence for the security of AI-based systems independent of modality or application: as long as proper threat modeling is carried out, one can readily employ \ac{RL} agents in order to devise optimal defenses, but only after they devised optimal attacks too.


%% file: appendix.tex
\appendices

\section{Proofs}
\label{apx:proofs}

For a more intuitive understanding of the proofs, we provide a high-level description of the attack fundamentals.
\textbf{BAGS} \cite{brunner2019guessing} performs a random walk along the boundary between the adversarial and the non-adversarial regions, by first taking a random step orthogonal to the original image direction, then a source step towards it.
The randomness in the directions searched is reduced by utilizing Perlin noise and masks computed on the difference between starting and original samples.
\textbf{HSJA} \cite{chen2020hopskipjumpattack} operates in 3 stages: a binary search that places the current best adversarial on the decision boundary, an estimation step that computes the gradient at that point of the boundary, and projection step along the estimated gradient.
These steps repeat until convergence.
\vspace{1em}

\textbf{[Proposition \ref{prop:one}]}
\vspace{-1em}
\begin{proof}
We proceed in two parts, corresponding to conditions (a) and (b).

\textbf{[Part 1]} Let us denote by $x_c, x_g, x_t$ the original (unperturbed), the starting, and the current sample at step $t$ respectively.
Given a target class $c_0 \in m$ we define a function:
\begin{equation}
    S_{x_c}(x_t) = F_{c_0}(x_t) - \underset{c\neq c_0}{\max}(F_c(x_t))
\label{eqn:S}
\end{equation}

HSJA operates in 3 stages that alternate until convergence: (1) binary search between $x_g$ and $x_c$ that places $x_t$ on the decision boundary, (2) gradient estimation approximating $\nabla S(x_t)$, (3) a step along the direction of the gradient $\nabla S(x_t)$. We repeat Eq. 9 of \cite{chen2020hopskipjumpattack}, denoting the gradient direction as the Monte Carlo estimate:

\begin{equation}
    \widetilde{\nabla S_{x_c}}(x_t,\delta) = \frac{1}{B} \sum_{b=1}^{B} \phi_{x_c}(x_t + \delta u_b)u_b
\label{eqn:nabla}
\end{equation}

where $\{u_b\}_{b=1}^{B}$ are i.i.d. draws from the uniform distribution over the $d$-dimensional sphere, $\delta$ is a small positive parameter, and $\phi_{x_c}$ is the Boolean-valued function that all stages rely on:
\begin{equation}
    \phi_{x_c}(x_t) = \text{sign}(S_{x_c}(x_t)) = 
    \begin{cases}
    +1 & \text{if} \quad S(x_t) > 0,\\
    -1 & \text{if} \quad S(x_t) \leq 0.
    \end{cases}
\label{eqn:boolphi}
\end{equation}

Given $x_c$, in search of adversarial examples HSJA iteratively applies the following update function:
\begin{equation}
    x_{t+1} = a_{t}x_c + (1-a_t)\Biggl\{x_t + \xi_t \frac{\nabla S_{x_c}(x_t)}{\|\nabla S_{x_c}(x_t)\|_2}\Biggl\}
\label{eqn:update}
\end{equation}

where $\xi_t$ is a positive step size and $a_t$ is a line search parameter in $[0,1]$ s.t. $S(x_{t+1}) = 0$, i.e. the next query lies on the boundary.
Now let us assume that the decision function $D$ is $\operatorname*{arg\,max}$, i.e. $D: \mathbb{R}^m \mapsto \mathbb{N}^m, \:C(x) = D(F_c(x)) = \operatorname*{arg\,max} F_c(x)$, then from Eq. \ref{eqn:classifier} and Eq. \ref{eqn:S} we have:
\begin{equation}
\begin{aligned}
    S_{x_c}(x_t) > 0 & \iff C(x_t) = c_0\\
    S_{x_c}(x_t) < 0 & \iff C(x_t) \neq c_0\\
    S_{x_c}(x_t) = 0 & \iff C(x_t) = \{c_0,a\},\:a\neq c_0\\
    \implies S_{x_c}(x_t) \leq 0 & \iff C(x_t) \neq c_0\\
\end{aligned}
\label{eqn:sc}
\end{equation}

Let us define the function $\mathcal{I}$ of two variables:
\begin{equation}
    \mathcal{I}(a,b) :=
    \begin{cases}
    +1 & \text{if} \quad a = b,\\
    -1 & \text{if} \quad a \neq b.
    \end{cases}
\label{eqn:id}
\end{equation}

From \ref{eqn:sc} and \ref{eqn:id} we can rewrite Eq. \ref{eqn:boolphi} as follows: 
\begin{equation}
    \phi_{x_c}(x_t) = \mathcal{I}(C(x_t), c_0)
\label{eqn:phi}
\end{equation}

Provided that the gradient estimation happens at the decision boundary where $S(x_t) = 0$, Theorem 2 of \cite{chen2020hopskipjumpattack} guarantees that the gradient estimation is an asymptotically unbiased direction of the true gradient:
\begin{equation}
  \widetilde{\nabla S_{x_c}}(x_t,\delta) \approx \nabla S_{x_c}(x_t), \delta \rightarrow 0
\label{eqn:approx}
\end{equation}

For $b_t = 1- a_t$ and by plugging \ref{eqn:phi} \& \ref{eqn:approx} in Eq. \ref{eqn:nabla}, and the result in \ref{eqn:update}, we get:
\begin{align}
    &x_{t+1} = a_{t}x_c + b_t\Biggl\{x_t + \xi_t \frac{\nabla S_{x_c}(x_t)}{\|\nabla S_{x_c}(x_t)\|_2}\Biggl\} \notag \\
    &= a_{t}x_c 
    + b_t\Biggl\{x_t + \xi_t \frac{\frac{1}{B} \sum_{b=1}^{B} \phi_{x_c}(x_t + \delta u_b)u_b}{\|\frac{1}{B} \sum_{b=1}^{B} \phi_{x_c}(x_t + \delta u_b)u_b\|_2}\Biggl\} \label{eqn:all} \\
    &= a_{t}x_c + b_t\Biggl\{x_t + \xi_t \frac{\frac{1}{B} \sum_{b=1}^{B} \mathcal{I}(C(x_t + \delta u_b), c_0)u_b}{\|\frac{1}{B} \sum_{b=1}^{B} \mathcal{I}(C(x_t + \delta u_b), c_0)u_b\|_2}\Biggl\} \notag
\end{align}

In Eq. \eqref{eqn:all}, the iterates $x_t$ are \textit{guaranteed} by Theorem 1 of HSJA \cite{chen2020hopskipjumpattack} to converge to a stationary point $x_b$ of Eq. \eqref{eqn:opt}, that is $\mathbb{E}[\sum_{i=1}^{N}d(x^i_b,x^i_c)] < \epsilon$, for $\epsilon$ a standard imperceptibility threshold and $N$ adversarial episodes, which contradicts the requirement $\mathbb{E}[\sum_i d(x_b^i,x_c^i)] \ge \epsilon$.
Since the only model-dependent term in Eq. \eqref{eqn:all} is the classifier $C(\cdot)$, the contradiction can be avoided only with an alternative classifier $C'$.
With $C(x)=D(F_c(x))$, and the discriminant function $F_c$ unable to change without retraining, it follows that $D'\neq \operatorname*{arg\,max}$, so for adv. example $x_t$ misclassified as $c_0$, Eq. \ref{eqn:psi} can return -1:

\begin{equation}
\begin{aligned}
    C(x_t) = c_0  \Rightarrow \psi(x_t) = -1 \\
    \therefore C(x_t) = \hat{c}, \: \hat{c} = \{c_0, m\setminus c_0, \varnothing\}
\end{aligned}
\end{equation}

\noindent where $\varnothing$ denotes rejection and $\{ m\setminus c_0\}$ denotes misdirection, i.e. intentional misclassification.

\textbf{[Part 2]} Let us assume that at timestep $t$, $x_t$ is not yet adversarial, it is however still \textit{part of} an ongoing adversarial attack.
To deter the attack, a perfect defense would have to misclassify/reject this example; yet if an identical but benign example $x_n$ was submitted, classifier $C$ should preserve its capacity to classify it correctly.
Since any memoryless classifier must assign the same label whenever $x_n = x_t$, we require a richer decision rule $C'(x, \tau) = D\bigl(\tau,\;F_c(x)\bigr)$ that takes as auxiliary input context $\tau$, e.g. the history of queries $\{x_0,\dots,x_t\} \cup x_n$.  By construction, even if \(x_n = x_t\), differing contexts \(\tau_n \neq \tau_t\) can force $C'(x_n,\tau_n) \;\neq\; C'(x_t,\tau_t)$, thereby separating benign from adversarial queries.
\label{prf:one}
\end{proof}

\textbf{[Proposition \ref{prop:two}]}
We annotate terms with $\mathcal{D}$ when an active defense $\pi_\phi^{\mathcal{D}}$ is present, and with $\cancel{\mathcal{D}}$ otherwise.
\begin{proof}
\textbf{BAGS}.
This attack is in effect a gradual interpolation from $x_g$ towards $x_c$, by first taking orthogonal steps $x_s$ on the hypersphere around $x_c$ and then source steps towards $x_c$ in order to minimize $d(x_c - x_b)$, where $x_b$ is the best adversarial example found so far.
The source step parameter $\epsilon = (1.3 - \min(\lambda_n, 1)) \cdot c$ -- with $\lambda_n$ the ratio of the $n$ last queries $x_t$ that are adversarial and $c$ a positive constant -- controls the projection towards $x_c$:
\begin{equation}
    x_t = x_s + \epsilon \cdot (x_c - x_s)
\label{eqn:source}
\end{equation}

Then if we again assume that a non-zero amount of the adversarial queries $x_t$ is flagged as such by the defense, it follows that $\lambda^{\cancel{\mathcal{D}}}_n > \lambda^\mathcal{D}_n$ and from the definition of $\epsilon$ we get $\epsilon^\mathcal{D} < \epsilon^{\cancel{\mathcal{D}}}$.
At given $t$, from Eq. \eqref{eqn:source} we get that $d(x_c,x^\mathcal{D}_t) > d(x_c, x^{\cancel{\mathcal{D}}}_t)$, and ceteris paribus the expectation \eqref{eqn:rew} will be larger with $\pi_\phi^{\mathcal{D}}$ present than without.

\textbf{HSJA}.
We denote the queries during gradient estimation as $x_n = x_t + \delta u$, $u \sim Uniform_{Sphere}(d)$, the ratio of those $x_n$ detected as adversarial by the active defense as $\eta \in [0,1]$, and the estimate $\widetilde{\nabla S_{x_c}}(x_t,\delta)$ as $u_t$.
We investigate the behavior of active defenses as the ratio of detections $\eta$ goes to 1.

For $\eta = 1 \implies \mathbb{E}[\phi_{x_c}(x_n)] = -1$, and as $u_b$ are uniformly distributed, from Eq. \eqref{eqn:nabla} we get:
\begin{equation}
\begin{aligned}
    \lim_{\eta\to1} u_t &= \lim_{\eta\to1} \frac{1}{B} \sum_{b=1}^{B} \phi_{x_c}(x_t + \delta u_b)u_b = \frac{1}{B} \sum_{b=1}^{B} -u_b
\label{eqn:limit}
\end{aligned}
\end{equation}

At the limit of detection we observe that the gradient estimate $u_t$ behaves like a uniformly drawn vector around $x_t$ of shrinking size.
By the Law of Large Numbers, as $B$ increases the average direction of $u_t$ will align with the expected value: that is a random direction on the unit hypersphere.
However, due to the $\frac{1}{B}$ term, the size of $u_t$ goes to 0.
From Eq. \eqref{eqn:limit} then we get: $\lim_{\eta\to1} u_t = 0$.
The gradient estimation step is followed by the ``jump'' step that computes $x_{t+1}$ as follows:
\begin{equation}
    x_{t+1} = x_t + \xi {u_t}
\end{equation}

As the ratio of detections $\eta$ approaches 1, we observe that the adversarial iterates $x_{t+1}$ converge prematurely: then all else being equal and for given $t$, $d(x_c,x^\mathcal{D}_t) > d(x_c, x^{\cancel{\mathcal{D}}}_t)$.

\end{proof}

\textbf{[Proposition \ref{th:epg}]}

\begin{proof}
Let $\pi_\theta^\mathcal{A}$ be the adversarial policy generating queries $x_t$, in $N$ episodes $\tau_i$ of length $L$.
The defense $\pi_\phi^{\mathcal{D}}$ \eqref{eqn:def}, upon receiving a query $x_t$ outputs a decision $\alpha_t = \pi_\phi^{\mathcal{D}}(x_t, s_t^\mathcal{D}) \in \{0,1\}$, with $1$ and $0$ indicating rejection and acceptance respectively.
The goal of the adversary is to find the parameters $\theta^*$ that maximize the expected reward:

\begin{equation}
\mathcal{J}(\theta) = \mathbb{E}_{\pi_\theta^\mathcal{A}} [r(\tau_i)] = \frac{1}{N}\sum_{i=1}^{N}r(\tau_i) = \frac{1}{N} \sum_{i=1}^{N} \sum_{t=1}^{L} (1 - \alpha_t)
\label{eqn:expect}
\end{equation}

The gradient of $\mathcal{J}(\theta)$ with respect to the policy parameters $\theta$ is the direction of steepest ascent for Eq. \eqref{eqn:expect}.
Through the Policy Gradient Theorem \cite{sutton1999policy} we can express the gradient of the expected reward in terms of the gradient of the log-likelihood of the policy, weighted by the reward:

\begin{equation}
\begin{aligned}
    \nabla_\theta \mathcal{J}(\theta) &= \nabla_\theta \mathbb{E}_{\pi_\theta^\mathcal{A}} [r(\tau_i)]\\
    &= \mathbb{E}_{\pi_\theta^{\mathcal{A}}}\left [r(\tau_i)\nabla_\theta \log \pi_\theta(\tau_i)\right]\\
    &= \mathbb{E}_{\pi_\theta^{\mathcal{A}}}\left[\sum_{t=1}^{L}(1 - a_t) \nabla_\theta \log \pi_\theta^{\mathcal{A}}(\tau_i)\right]\\
    &\approx \frac{1}{N}\sum_{i=1}^{N}\sum_{t=1}^{L}(1 - a_t) \nabla_\theta \log \pi_\theta^{\mathcal{A}}(\tau_i)\\
\label{eqn:rew_app}
\end{aligned}
\end{equation}
The gradient is thus estimated by sampling $N$ episodes from the policy $\pi_\theta^\mathcal{A}$ to compute Eq. \eqref{eqn:rew_app}.
To maximize the expectation, we iteratively update the policy parameters $\theta$ using gradient ascent: $\theta \leftarrow \theta + \eta \nabla_\theta \mathcal{J}(\theta)$.
With $\eta > 0$ the learning rate, this process converges to $\theta^*$.
Recall that Eq. \eqref{eqn:expect} attains its maximum for $\sum_{i=1}^{N}\sum_{t=1}^{L}a_t = 0$, therefore the converged policy $\pi^\mathcal{A}_{\theta^*}$ will correspond to the optimal evasive policy.

\end{proof}

\section{On States \& Rewards}
\label{apx:rew}
\textbf{[States]}.
To handle the partial observability, we engineer states that incorporate past information.
For \textbf{BAGS}, the adversary uses an 8-dimensional state representation with the following information normalized in the range $[0,1]$ : current amount of queries $i$, average queries that are adversarial $a$, the initial gap $g$, the current gap $d$, the location $l = \frac{d}{g}$, the slope $s = m - l$ where $m$ is a moving average of the location, the frequency of improvement $f$, and $r$ which is a moving average of the perturbation reduction $n$, 
In \textbf{HSJA} the state representation is slightly different: $r = \frac{n}{g}$, and $f = \frac{1}{j}$ with $j$ number of jump steps in last iteration.

For the defense, in \textbf{HSJA} (and to a lesser extent \textbf{BAGS}) it has been difficult to engineer a state for policies to effectively learn on.
The knowledge of the attack internals and fundamentals, geometric properties and distances, model activations and logits, and any combination thereof, did not suffice.
Ultimately we decided to learn a representation instead.
This representation is a 64-dimensional embedding of a CNN trained with triplet loss, on data generated by HSJA and benign queries, with the input being a tensor of the last query subtracted from the 25 most recent adversarial queries and then stacked.

\textbf{[Rewards]}. The concrete definitions of the rewards for each type of agent are:

\begin{itemize}[leftmargin=*]
    \setlength\itemsep{0.5em}
    \item BAGS adversary: with $x \in [1,50]$ the number of queries to a better adversarial example and $t$ the maximum queries: $\textbf{R1} = \frac{n\cdot x}{g} $ if $n > 0$ else 0 \text{\textbar} $\textbf{R2} = \frac{n}{g\cdot(x+1)} $ if $n > 0$ else 0 \text{\textbar} $\textbf{R3} = (1-\sqrt{\frac{d}{g}})^2 - (1-\sqrt{\frac{d+n}{g}})^2$ \text{\textbar} $\textbf{R4} = \sqrt{i} \cdot R2$ \text{\textbar} $\textbf{R5} = |\log(d/g)|$ if $i\geq t$ else 0 \text{\textbar} $\textbf{R6} = \sqrt[4]{i} \cdot a$ \text{\textbar} $\textbf{R7} = R4 + R6$.
    \item HSJA adversary: with $e$ the gradient estimation steps: $\textbf{R1} = 2 \cdot n$ \text{\textbar} $\textbf{R2} = \frac{-e}{1000} + R1$ \text{\textbar} $\textbf{R3} = \frac{10\cdot n}{d} $ \text{\textbar} $\textbf{R4} = \frac{1}{d}$ \text{\textbar} $\textbf{R5} = \frac{2\cdot (g - d)}{g}$ if $i\geq t$ else 0 \text{\textbar} $\textbf{R6} = 2\cdot (0.5 - |\frac{a+1}{2} - 0.5|) + b$, where $b = \frac{j}{20}$ if $j<3$ else 0 \text{\textbar} $\textbf{R7} = R3 + R6$ \text{\textbar} $\textbf{R8} = R5 + R6$.
    \item BAGS defender: where $x_g$ is the starting sample, $x_t$ the last query, $x_b$ the best adversarial so far, $s_t$ the average step size between queries, $h \in [0,1]$ the last action of the defender, $z \in [0,1]$ the $\ell_2$ distance of $x_t$ and the last known adversarial query in embedding space, $x$: $\textbf{R1} = |\log(0.1 g + \lVert x_g, x_b \rVert _{\ell_2})| \cdot 0.1$ \text{\textbar} $\textbf{R2} = |\log_{10}s_t| $ \text{\textbar} $\textbf{R3} = \frac{g}{\lVert x_g, x_t \rVert}$ \text{\textbar} $\textbf{R4} = - \psi (x_t)$, where $\psi$ is Eq. \ref{eqn:psi} \text{\textbar} $\textbf{R5} = h - z$.
    \item HSJA defender: where $x_{BS}$ are queries during the binary search: $\textbf{R1} = 1 - 2(\frac{\lVert x_g, x_b \rVert}{g})$ \text{\textbar} $\textbf{R2} = h - z$ \text{\textbar} $\textbf{R3} = R2 - 2\psi (x_{BS})$ \text{\textbar} $\textbf{R4} = - \lVert \psi(x_{BS}) \rVert$ \text{\textbar} $\textbf{R5} = R2$ if $\psi(x_t)$ else $2\cdot R2$.
    \item For both BAGS and HSJA defenders, the aforementioned are the rewards when $x_t$ is adversarial; when it is benign, the reward is $R = 1 - h$ if the model responded correctly, otherwise $R = -1$.
\end{itemize}

\begin{table}[ht!]
\centering
\renewcommand*{\arraystretch}{1.05}
\caption{Input Transformations.}
\begin{tabular}{r|r|r}
\toprule
\textbf{Input Transformations} &\bf Magnitude &\bf Probability\\
\midrule
Brightness \& Contrast & 0 -- 0.5 & 0 -- 1\\
Random Horizontal Flip & -- & 0 -- 1\\
Random Vertical Flip & -- & 0 -- 1\\
Sharpness & 0.8 -- 1.8 & 0 -- 1 \\
Perspective & 0.25 -- 0.5 & 0 -- 1 \\
Rotation & \textdegree 0 -- \textdegree 180 & 0 -- 1\\
Uniform Pixel Scale & 0.8 -- 1.2 & 0 -- 1 \\
Crop \& Resize & 0.6 -- 1 & 0 -- 1\\
Translation & -0.2 -- 0.2 & 0 -- 1\\
\bottomrule
\end{tabular}
\label{tbl:transforms}
\end{table}

\section{Base Rate of Attacks}
\label{baserate}


In all the evaluations so far we use a fixed probability $P(adv)=0.5$ that an incoming query is adversarial.
To assess how our adaptive stateful defenses (Scenarios 4 \& 6) perform in the complete absence of attacks, we evaluate them with $P(adv)=0$ \textit{without} retraining; the results are shown in Table~\ref{tbl:zero}.
We observe a small reduction in the accuracy on clean samples that can be attributed to the considerably different base rate of adversarial and benign queries.
Note however that as the probability of adversarial queries is an intrinsic property of each environment, if the base rate of attacks changes the defensive agents can be retrained to adjust to it.

\begin{table}[h]
\small
\caption{Clean accuracy on CIFAR-10 for scenarios 4 \& 6, where $P(adv)$ denotes the probability that a query is part of an attack.}
\centering
\renewcommand*{\arraystretch}{1.05}
\begin{tabular}{c|r|r|r|r|r|}
\toprule
\multirow{2}{*}{\parbox{1cm}{\centering Adv.\\Trained}} & \multirow{2}{*}{$P(adv)$} & \multirow{2}{*}{BAGS4} & \multirow{2}{*}{BAGS6} & \multirow{2}{*}{HSJA4} & \multirow{2}{*}{HSJA6} \\
& & & & & \\
\midrule
\multirow{2}{*}{\xmark} & 0.5 & 91.55 & 91.38 & 91.59 & 91.62 \\
& 0.0 & 90.91 & 90.95 & 90.48 & 90.86 \\
\midrule
\multirow{2}{*}{\cmark} & 0.5 & 87.61 & 87.50 & 87.58 & 87.68 \\
& 0.0 & 87.02 & 87.11 & 86.70 & 86.98 \\
\bottomrule
\end{tabular}
\label{tbl:zero}
\end{table}

\section{Models \& Hyperparameters}
\label{apx:hyper}
The image classification models we use are ResNet-20 for CIFAR-10 and a standard 2 convolutional / 2 fully-connected layer NN for MNIST.
For adversarially training models, we follow the canonical approach as described in \cite{wang2019convergence}: the model is trained for 20 epochs, where the first 10 are trained normally and the last 10 on batches containing additional adversarial examples generated with 40 steps of PGD.
For learning the similarity space, that is the metric space where defensive agents control the radius of interception around which a query is adversarial or not, we use a Siamese CNN.
This network is trained with contrastive loss, where dissimilar examples are generated by adding Gaussian noise and performing evasive transformations on the input from the list in Table \ref{tbl:transforms}.
For the PPO agents trained for each scenario, we use the open source library Stable-Baselines3 \footnote{https://github.com/DLR-RM/stable-baselines3}.
Policies are parameterized by a two fully-connected layer NN; the hyperparameter search space is shown in Table \ref{tbl:agents}.

\begin{figure*}[]
    \centering
    \includegraphics[width=0.7\textwidth]{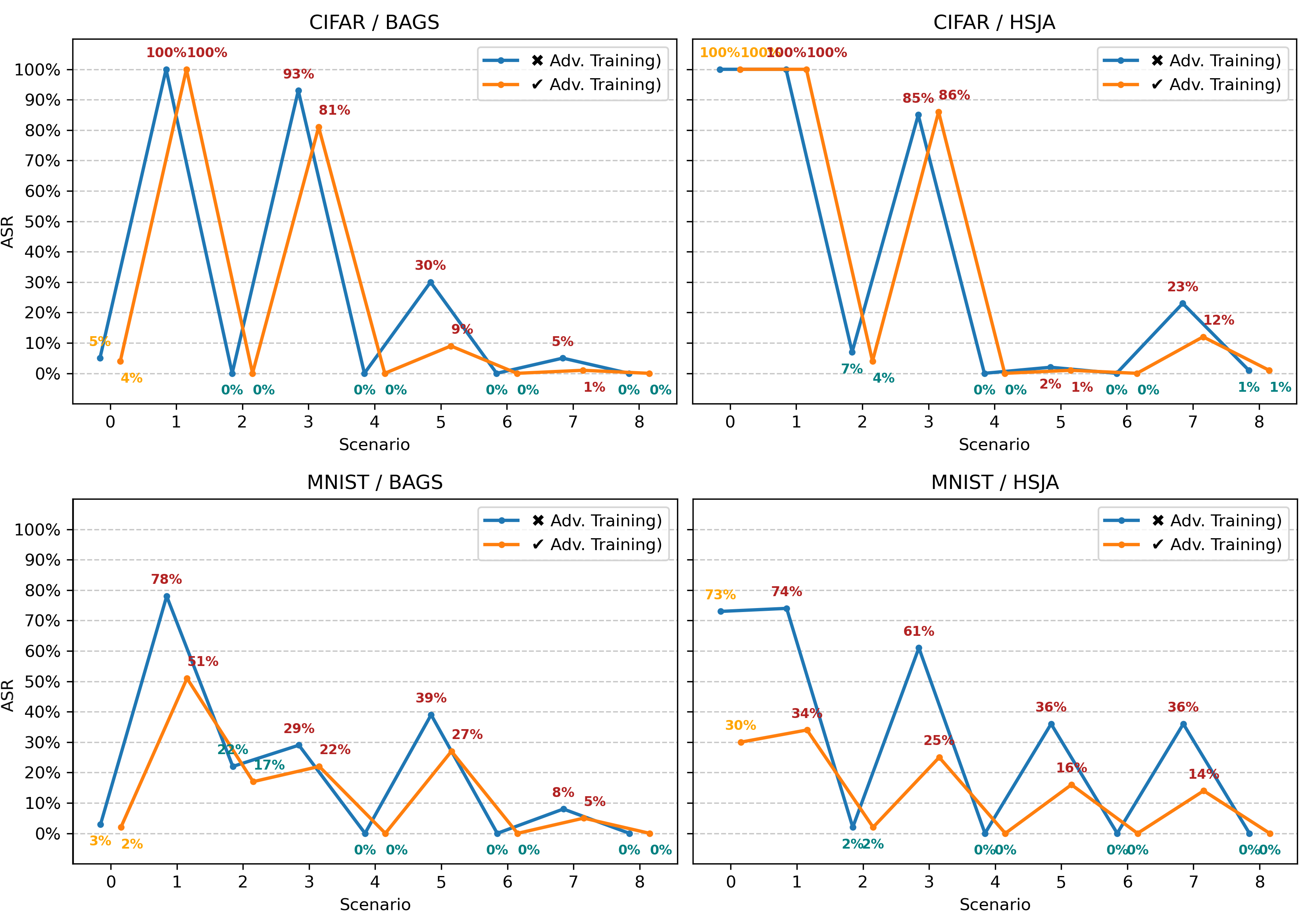}
    \caption{\small{Progression of ASR over successive adaptations. Red and green values in ASR denote offensive and defensive scenarios respectively.}}
    \label{fig:plot}
\end{figure*}

\textbf{Blacklight \& OARS}. For evaluating Blacklight \cite{li2022blacklight} and OARS \cite{feng2023stateful} we use their default hyperparameters as those are provided in the publicly available implementations.
In particular, as OARS spends 200 extra queries per episode to adapt the proposal distribution, we add those on top of b the evaluation budget.
Additionally, as our defense is not rejection based, we replace the rejection decision with a non-adversarial one.

\begin{table}[h]
\centering
\renewcommand*{\arraystretch}{1.05}
\caption{Hyperparameter ranges during PPO training.}
\begin{tabular}{|r|r|r|}
\toprule
\textbf{Hyperparameter} &\bf BAGS &\bf HSJA\\
\midrule
learning rate & 3e-3 -- 1e-4 & 3e-3 -- 1e-4 \\
episode steps & 600 -- 3000 & 1000 -- 5000 \\
total steps & 1e5 -- 4e5 & 2e4 -- 2e5\\
total queries & 25K -- 100K & 5K -- 50K\\
batch size & 32 -- 128 & 32 -- 64\\
buffer size & 2048 -- 2048 & 64 -- 1024 \\
epochs & 20 -- 20 & 20 -- 20 \\
gamma & 0.85 -- 0.99 & 0.9 -- 0.99\\
\bottomrule
\end{tabular}
\label{tbl:agents}
\end{table}